\documentclass[letterpaper]{article} \usepackage{aaai2026}  \usepackage{times}  \usepackage{helvet}  \usepackage{courier}  \usepackage[hyphens]{url}  \usepackage{graphicx}    \usepackage{natbib}  \usepackage{caption}       \usepackage{algorithm}
\usepackage{algorithmic}

\usepackage{amsmath}
\usepackage{booktabs}
\usepackage{tabularray}
\UseTblrLibrary{booktabs}

\usepackage{newfloat}
\usepackage{listings}
\usepackage{subcaption}
\usepackage{trimclip}
\usepackage[table]{xcolor}
\usepackage{multirow}
\usepackage{amsthm}
\usepackage{tabularx}
\DeclareCaptionStyle{ruled}{labelfont=normalfont,labelsep=colon,strut=off} 
\floatstyle{ruled}
\newfloat{listing}{tb}{lst}{}
\floatname{listing}{Listing}
\usepackage{amsmath,amsfonts,bm}

\def\eqref#1{equation~\ref{#1}}

\def\1{\bm{1}}

\DeclareMathAlphabet{\mathsfit}{\encodingdefault}{\sfdefault}{m}{sl}
\SetMathAlphabet{\mathsfit}{bold}{\encodingdefault}{\sfdefault}{bx}{n}

\theoremstyle{plain}\newtheorem{theorem}{Theorem}\newtheorem{proposition}[theorem]{Proposition}

\theoremstyle{remark}

\theoremstyle{definition}

\let\titleold\title
\renewcommand{\title}[1]{\titleold{#1}\newcommand{\thetitle}{#1}}

\title{From Parameter to Representation: A Closed-Form Approach for\\Controllable Model Merging}
\author {
Jialin Wu\textsuperscript{\rm 1},
    Jian Yang\textsuperscript{\rm 2},
    Handing Wang\textsuperscript{\rm 3},
    Jiajun Wen\textsuperscript{\rm 1},
    Zhiyong Yu\textsuperscript{\rm 1}\thanks{Corresponding author.}
}
\affiliations {
\textsuperscript{\rm 1}Department of Computer, Rocket Force University of Engineering\\
    \textsuperscript{\rm 2}Department of Engineering, Rocket Force University of Engineering\\
    \textsuperscript{\rm 3}School of Artificial Intelligence, Xidian University\\
    
wujialin11@nudt.edu.cn, yangjian@nudt.edu.cn, 
    hdwang@xidian.edu.cn,
    wenjiajun11@nudt.edu.cn,  yutouzy@163.com
}

\begin{document}

\maketitle

\begin{abstract}

Model merging combines expert models for multitask performance but faces challenges from parameter interference. This has sparked recent interest in controllable model merging, giving users the ability to explicitly balance performance trade-offs. Existing approaches employ a compile-then-query paradigm, performing a costly offline multi-objective optimization to enable fast, preference-aware model generation. This offline stage typically involves iterative search or dedicated training, with complexity that grows exponentially with the number of tasks. To overcome these limitations, we shift the perspective from parameter-space optimization to a direct correction of the model's final representation. Our approach models this correction as an optimal linear transformation, yielding a closed-form solution that replaces the entire offline optimization process with a single-step, architecture-agnostic computation. This solution directly incorporates user preferences, allowing a Pareto-optimal model to be generated on-the-fly with complexity that scales linearly with the number of tasks. Experimental results show our method generates a superior Pareto front with more precise preference alignment and drastically reduced computational cost.

\end{abstract}

\section{Introduction}

The shift towards fine-tuning large pre-trained models has led to a surge in specialized models customized for specific tasks. Merging these models has become an effective way to combine their strengths into a single, versatile network without needing expensive retraining or access to the original training data. Early approaches, like  averaging weights~\citep{DBLP:conf/nips/MatenaR22, DBLP:conf/iclr/Jin0P023} or performing task vector arithmetic~\citep{DBLP:conf/iclr/IlharcoRWSHF23, DBLP:conf/nips/YadavTCRB23, DBLP:conf/iclr/YangW00G0T24}, aimed to produce a static, merged model. However, these \textit{one-size-fits-all} methods often suffer from interference between task-specific parameters, leaving its performance far behind that of the individual experts. This revealed a need for a more flexible and controllable approach to model merging, letting users explicitly adjust the balance based on their preferences.

\begin{figure}[tb]
    \centering
    \clipbox{0.2cm 0.2cm 0.2cm 0.2cm}{\includegraphics[width=0.5\textwidth]{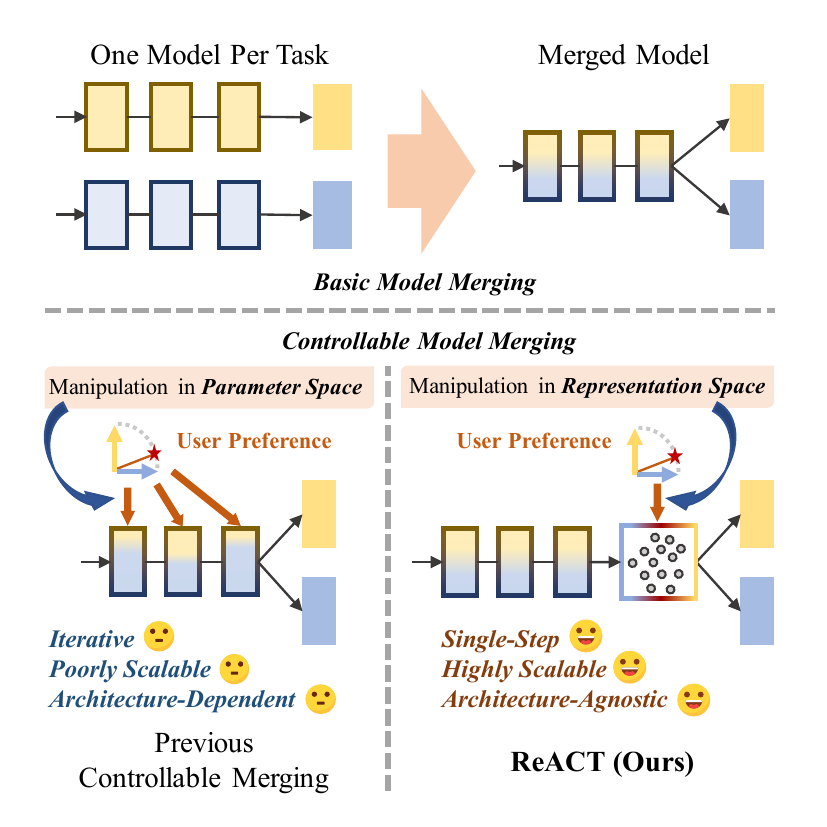}}
    \caption{Conceptual distinction between basic model merging (top) and controllable merging (bottom), exemplified with a two-task scenario. While prior approaches to controllable merging (bottom left) rely on slow, iterative optimization in parameter space, our method (bottom right) achieves direct control through an efficient, single-step correction in representation space.}
    \label{fig:concept}
\end{figure}

Recent methods like Pareto Merging (PM)~\citep{chenParetoMerging2025} and MAP~\citep{liMAPLowcomputeModel2024} enable user control by framing model merging as a multi-objective optimization (MOO) problem. Their shared \textit{compile-then-query} paradigm, however, imposes severe overhead in computation, memory, and data. The offline \textit{compiling} stage is computationally prohibitive: PM requires complex iterative training that can become unstable, while MAP employs an evolutionary search whose complexity grows exponentially with more tasks. This cost is further magnified when combining PM with advanced backbones like AdaMerging~\citep{DBLP:conf/iclr/YangW00G0T24}, which necessitates a more complex joint optimization. In terms of resources, MAP must store all original task vectors, a memory requirement that scales poorly, and often relies on labeled data which may be unavailable. Moreover, this entire costly preparation of both MAP and PM must be repeated from scratch whenever the set of expert models changes, hindering practical scalability and adaptation.

To address these concerns, we propose \textbf{ReACT} (\textbf{Re}presentation \textbf{A}nalytical \textbf{C}ontrol \textbf{T}ransformation), a fundamentally different, \textit{on-the-fly} analytical approach.
As conceptually illustrated in Figure~\ref{fig:concept}, ReACT shift the paradigm from costly parameter-space optimization to a direct correction in the model's final representation space, based on our key insight that performance degradation stems from a global linear distortion. We therefore reframe the problem as finding an optimal linear correction map, regularized by an orthogonal prior. Crucially, this formulation admits a closed-form solution that naturally extends to the multi-objective case via linear scalarization, bypassing iterative optimization. The resulting analytical framework is inherently architecture-agnostic and highly modular, allowing expert models to be added or removed without costly re-optimization. ReACT enables on-the-fly generation of preference-aware models for real-time exploration and is highly data-efficient, achieving near-best performance with only a fraction of unlabeled test data. Our main contributions are threefold:
\begin{itemize}
    \item We reframe controllable merging as a representation correction problem, identifying the primary bottleneck as a simple linear distortion rather than a complex parameter conflict.
    \item We introduce an \textit{on-the-fly} analytical method, deriving what is, to our knowledge, the first \textit{closed-form solution} for controllable model merging that bypasses iterative optimization entirely.
    \item Extensive experiments show our method achieves a state-of-the-art Pareto front, superior preference alignment, and drastically reduced computational cost, maintaining this strong performance even when using only a fraction of the unlabeled data required by competing methods.
\end{itemize}

\section{Related Work}

\paragraph{Model Merging.} 
Model merging aims to consolidate multiple task-specific models, fine-tuned from a common pre-trained base, into a single network. We categorize prior work based on the nature of the merged model they produce.
A primary line of work focuses on creating a single, static merged model. This includes simple weight averaging~\citep{DBLP:conf/nips/SinghJ20, DBLP:conf/nips/MatenaR22, DBLP:conf/iclr/Jin0P023} and more sophisticated methods operating on task vectors—the parameter shifts from pre-training~\citep{DBLP:conf/iclr/IlharcoRWSHF23}. Subsequent refinements have focused on mitigating parameter conflicts through techniques like pruning and rescaling~\citep{DBLP:conf/nips/YadavTCRB23, DBLP:conf/iclr/YangW00G0T24, DBLP:conf/icml/Yu0Y0L24, DBLP:conf/nips/0002LLJGYLGTH024}. However, these methods produce a \textit{one-size-fits-all} model, offering no mechanism to control performance trade-offs among tasks.
Another category introduces dynamic, task-specific modules to improve performance. For instance, Representation Surgery~\citep{yangRepresentationSurgeryMultiTask2024} learns a lightweight MLP to correct representation bias for each task, while other methods generate task-specific masks or adapters~\citep{DBLP:conf/icml/WangDOFF24, DBLP:conf/nips/HuangY000O24, DBLP:conf/nips/LuF0QC024, qiLessMoreEfficient2025}. While effective, these approaches typically require loading a single task's module during inference, preventing a smooth, controllable trade-off across multiple tasks simultaneously.

The works most relevant to ours are those that enable controllable merging, often framing it as a Multi-Objective Optimization (MOO) problem. Pareto Merging (PM)~\citep{chenParetoMerging2025} learns a low-rank representation of the Pareto set via a complex optimization, while MAP~\citep{liMAPLowcomputeModel2024} employs an evolutionary algorithm to search for solutions. Both pioneering methods follow a \textit{compile-then-query} paradigm that operates in the parameter space. Despite their innovation, this paradigm imposes a significant upfront computational cost and suffers from poor scalability as the number of tasks increases. In contrast, our work proposes an \textit{on-the-fly} analytical framework that operates in the representation space.

\paragraph{Multi-Objective Optimization in Deep Learning.}
Multi-Objective Optimization (MOO) is a foundational framework for balancing competing objectives, not only in model merging but also in the broader field of multi-task learning (MTL). In the context of MTL, a dominant line of work focuses on training-based methods that operate during training or train a preference-aware module directly, such as GradNorm~\citep{DBLP:conf/icml/ChenBLR18}, PCGrad~\citep{DBLP:conf/nips/YuK0LHF20} and PaLoRA~\citep{DBLP:conf/iclr/DimitriadisFF25}. These approaches iteratively manipulate task gradients or sample preference vectors at each training step to steer the model towards the Pareto front.

However, a significant drawback of these training-time methods is their computational expense and reliance on the full training pipeline. Our work, in contrast, operates in a distinct, post-hoc paradigm. We apply MOO principles to already-trained models, sidestepping the costly iterative training process entirely. Furthermore, by formulating the problem in the representation space, we derive a closed-form solution, eliminating even the post-hoc optimization or search required by other controllable merging methods. 

\begin{figure*}[ht]
    \centering
    \includegraphics[width=\textwidth]{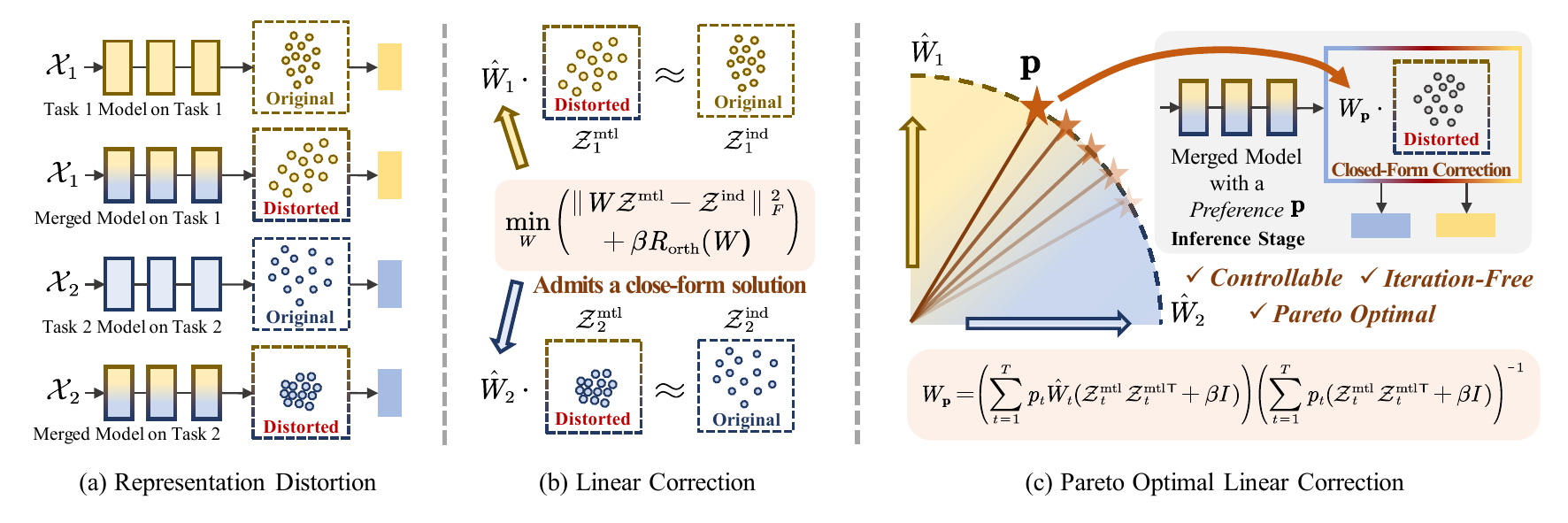}
    \caption{An overview of our proposed method,  illustrated with a two-task~($T=2$) example.  (a) We first identify that model merging causes representation distortion: the feature distribution of a merged model deviates from that of an individual task model. (b) We propose to correct this by finding a linear correction matrix $\hat{W}_t$ for each task, which has a closed-form solution. (c) Finally, we derive a Pareto-optimal transformation $W_\mathbf{p}$ by analytically aggregating the individual corrections based on a user preference vector $\mathbf{p}$. }

    \label{fig:overview}
\end{figure*}

\section{Preliminaries}
\subsection{Model Merging and Task Vectors}
Model merging techniques typically operate on a set of $T$ task-specific models, $\{\theta_1, \theta_2, ..., \theta_T\}$, that are fine-tuned from a common pre-trained model, $\theta_0$. A powerful concept in this area is the \textit{task vector}~\citep{DBLP:conf/iclr/IlharcoRWSHF23}, which captures the parameter shift for a specific task:
\begin{equation}
    \Delta_t = \theta_t - \theta_0
\end{equation}
A simple merged model, $\theta_{\text{merge}}$, can then be constructed by applying a linear combination of these task vectors to the pre-trained base:
\begin{equation}
    \theta_{\text{merge}} = \theta_0 + \sum_{t=1}^{T} \alpha_t \Delta_t
    \label{eq:task_vector}
\end{equation}
where $\alpha_t$ are merging coefficients. More advanced methods determine these coefficients $\alpha_t$ sophistically, sometimes even optimizing them layer-wise or per-parameter. The pre-merged model $\theta_{\text{merge}}$ can be obtained via various methods, from the simple linear combination in Eq.~\ref{eq:task_vector} to more advanced techniques like TIES-Merging or AdaMerging. Our framework operates as a post-hoc correction on any such pre-merged model.

\subsection{Controllable Merging as MOO}
The goal of controllable merging is to find solutions that optimally trade off performance across multiple tasks. This can be naturally formulated as a Multi-Objective Optimization (MOO) problem. Given $T$ tasks, we aim to simultaneously minimize a vector of loss functions, $\{L_1(\theta), L_2(\theta), ..., L_T(\theta)\}$, where $L_t(\theta)$ is the loss for task $t$ using model parameters $\theta$.

A solution $\theta^*$ is considered \textit{Pareto-optimal} if it is not possible to improve performance on one task without degrading performance on at least one other task. The set of all such Pareto-optimal solutions forms the \textit{Pareto front}~\citep{fleischer2003measure}.

A common technique to find points on the Pareto front is \textit{Linear Scalarization}. This method transforms the multi-objective problem into a single-objective one by taking a weighted sum of the individual objectives, guided by a user-defined preference vector $\mathbf{p} = [p_1, ..., p_T]^\mathsf{T}$, where $p_t \ge 0$ and $\sum p_t = 1$:
\begin{equation}
    \min_{\theta} \sum_{t=1}^{T} p_t L_t(\theta)
\end{equation}
Crucially, existing methods attempt to solve this MOO problem directly in the high-dimensional parameter space of $\theta$. In contrast, our work will reformulate the objective in the model's final representation space, leading to a more direct and efficient solution.

\section{Method}
Instead of performing expensive optimization in the parameter space, our method, ReACT, directly corrects the model's final representation. For a given task $t$, our approach is founded on the hypothesis that the discrepancy between the representations from the merged model, $\mathcal{Z}^\text{mtl}_t$, and those from the ideal single-task expert, $\mathcal{Z}^\text{ind}_t$, is primarily a linear distortion. This allows us to model the correction as a simple linear transformation. We show that finding the optimal transformation is a multi-objective problem for which we can derive a closed-form, Pareto-optimal solution for any user preference. Figure~\ref{fig:overview} provides an overview of our framework. The following sections first formalize this correction for a single task and then extend it to the multi-objective setting.

\begin{figure*}[ht]
    \centering
    \begin{subfigure}[t]{0.24\textwidth}
    \centering
    \includegraphics[width=\textwidth]{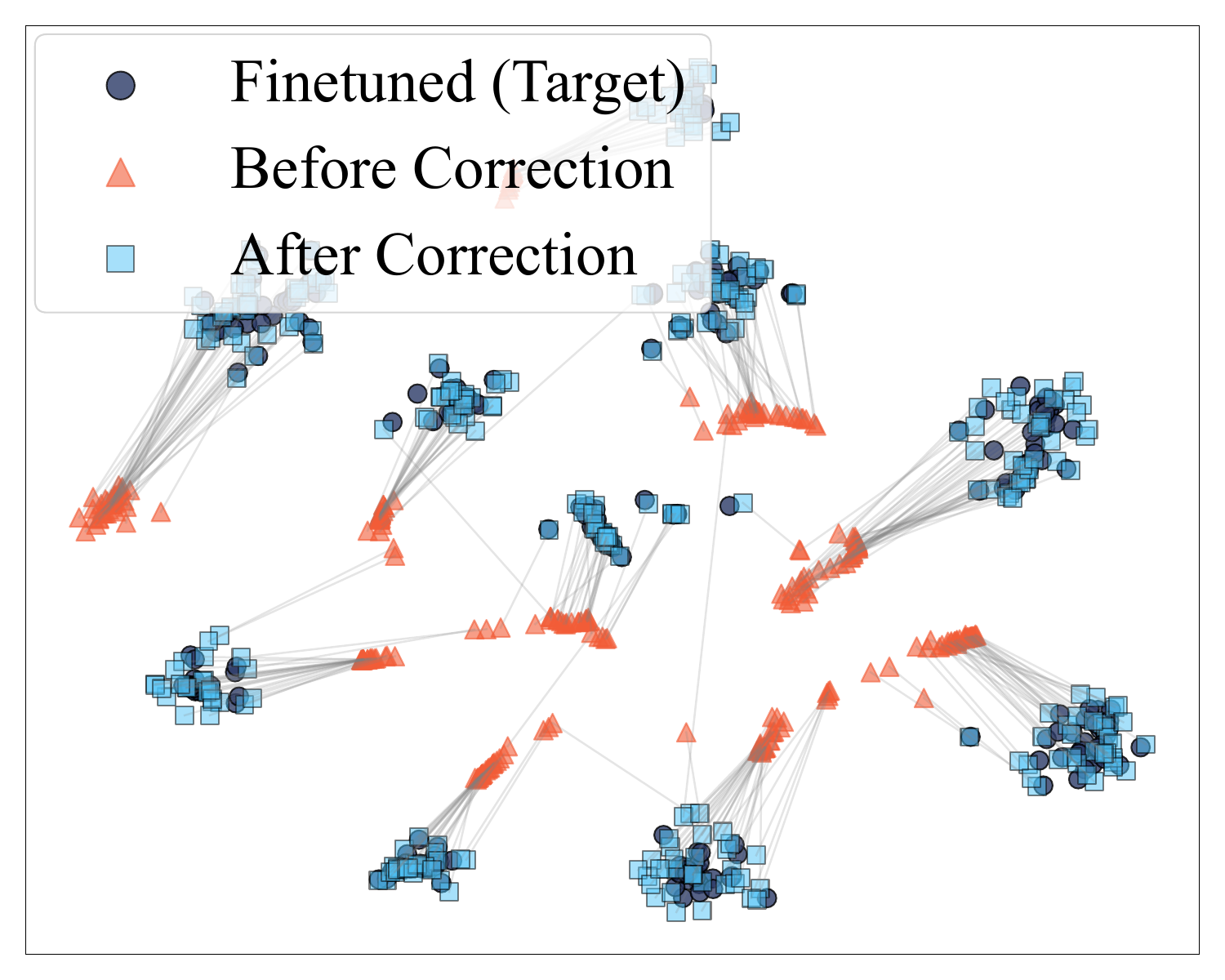}
    \caption{MNIST}
    \hfill
    \end{subfigure}
    \begin{subfigure}[t]{0.24\textwidth}
        \centering
        \includegraphics[width=\textwidth]{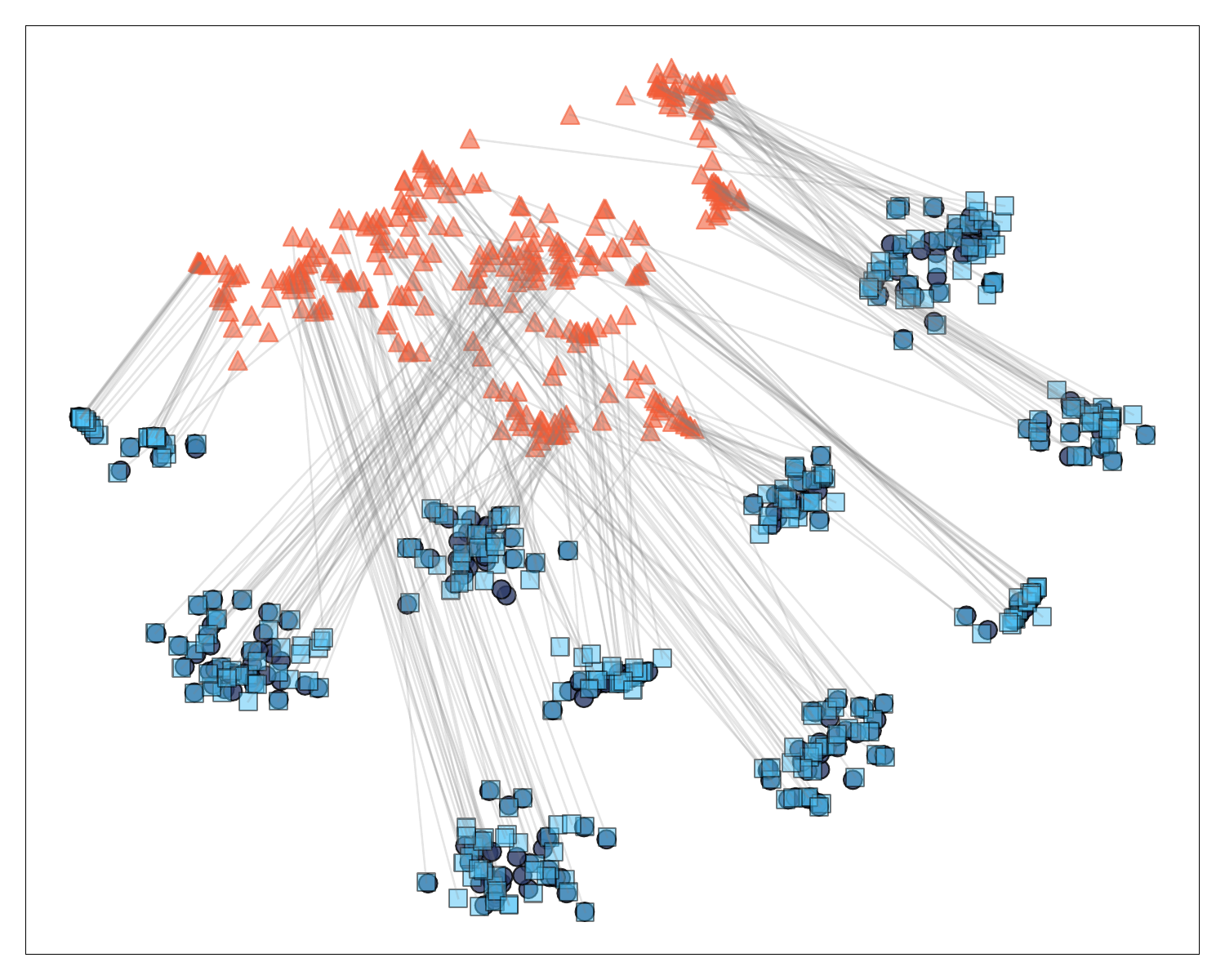}
        \caption{EuroSAT}
    \end{subfigure} 
    \hfill
    \begin{subfigure}[t]{0.24\textwidth}
        \centering
        \includegraphics[width=\textwidth]{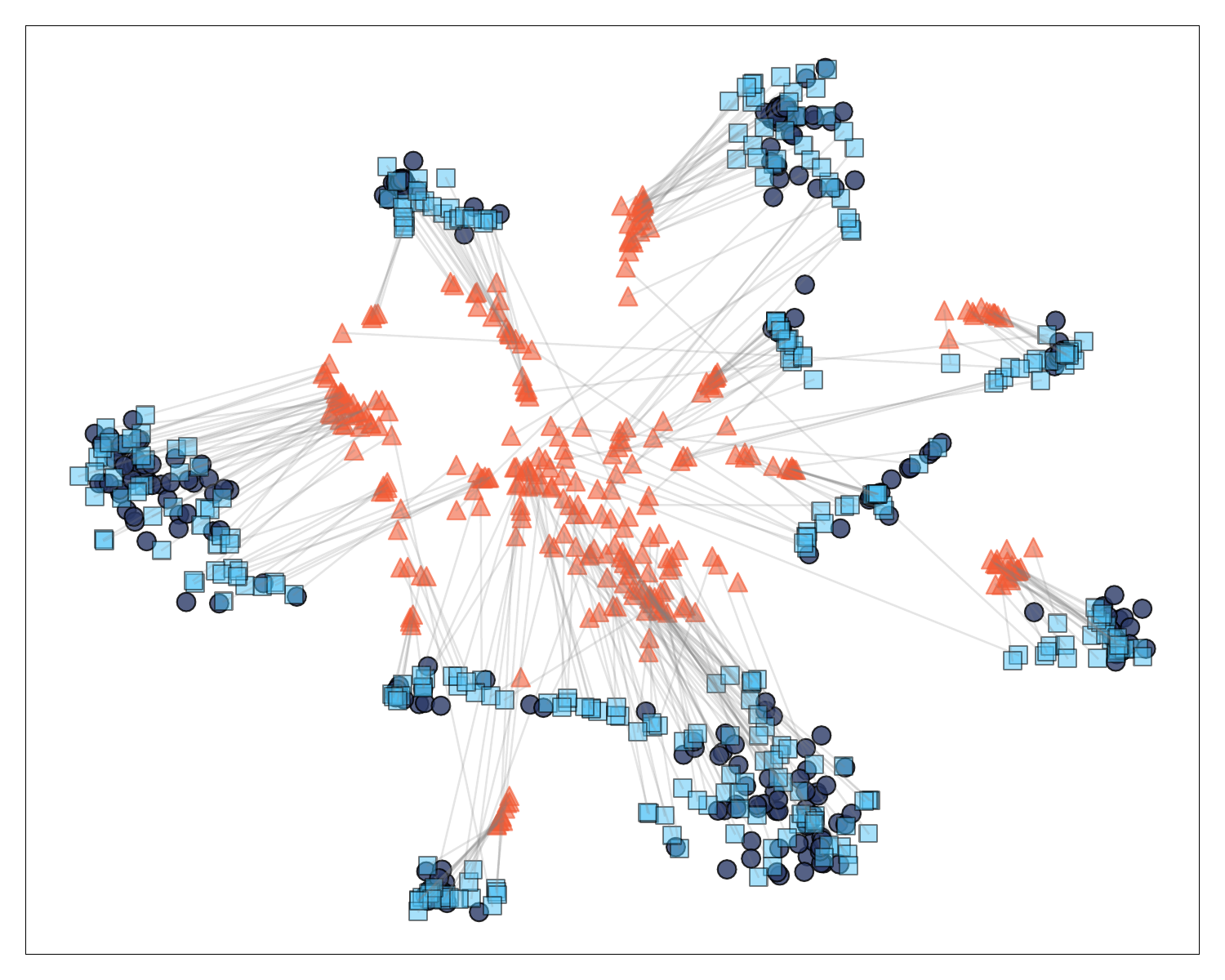}
        \caption{SVHN}
    \end{subfigure}
    \hfill
    \begin{subfigure}[t]{0.24\textwidth}
        \centering
        \includegraphics[width=\textwidth]{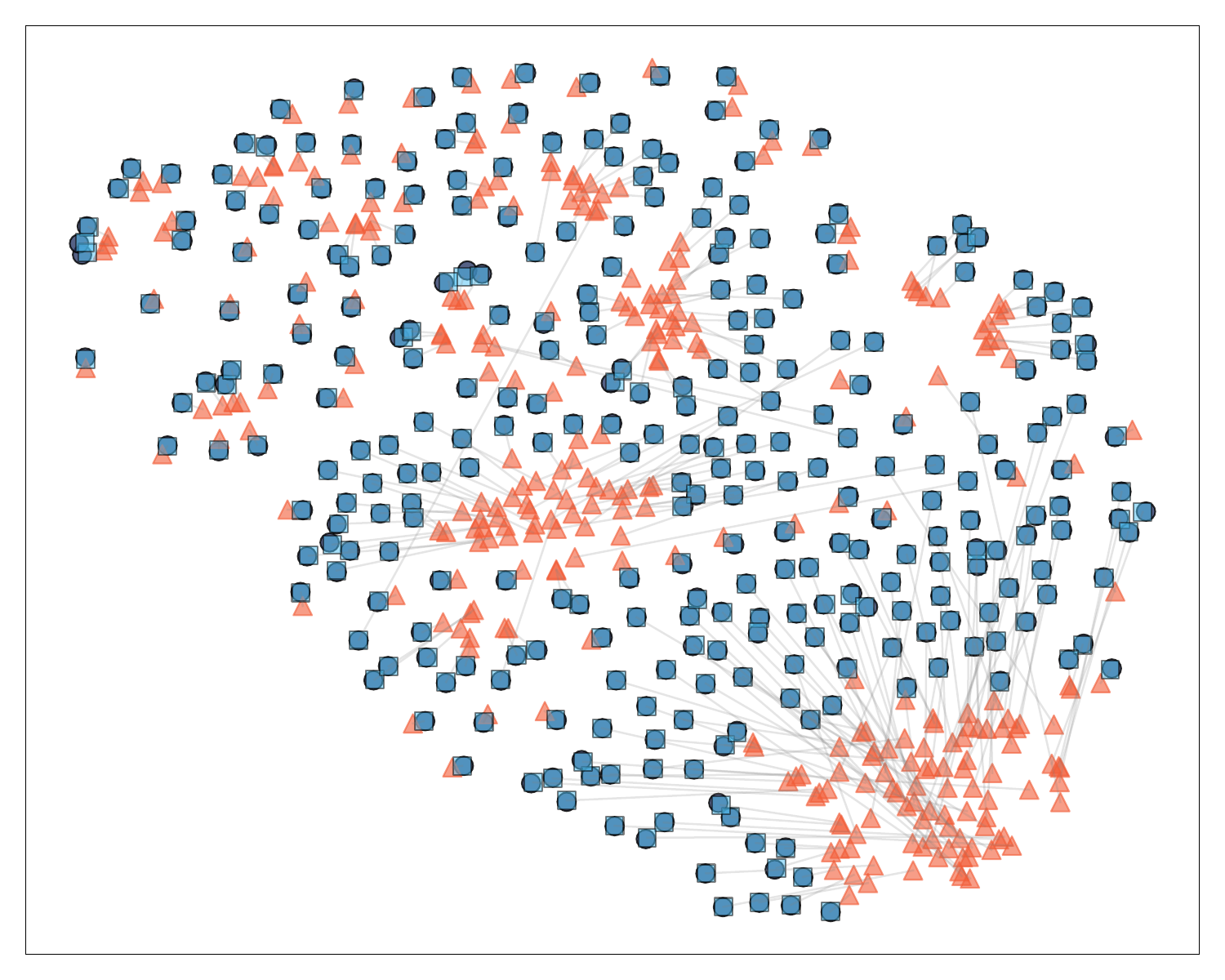}
        \caption{Cars}
    \end{subfigure}
    \caption{Visualization of representation correction. t-SNE plots for four tasks from an 8-model merge show the merged model's representations (orange) are severely misaligned with ideal single-task targets (dark blue). Our method effectively pulls the corrected features (light blue) back to the target clusters, visually confirming that a simple linear transformation is sufficient to correct the discrepancy. More t-SNE plots see Appendix D.5.}
    \label{fig:visual_correction}
\end{figure*}

\subsection{Linear Representation Correction}
Prior work like Representation Surgery~\citep{yangRepresentationSurgeryMultiTask2024} employs sample-specific non-linear functions for this correction. However, as hypothesized and visualized in Figure~\ref{fig:visual_correction}, this misalignment between $\mathcal{Z}^\text{mtl}_t$ and $\mathcal{Z}^\text{ind}_t$ is structurally simple. We therefore hypothesize it is dominated by a global linear distortion (e.g., rotation and scaling) rather than a sample-specific non-linear warp, as shown in Figure~\ref{fig:visual_correction}.
Based on this key insight, we propose a simpler, more elegant correction. We model the correction as a linear transformation matrix $W_t$ for each task $t$:
\begin{equation}
\mathcal{\hat{Z}}^\text{mtl}_{t} = W_t  \mathcal{Z}^\text{mtl}_{t}
\label{eq:correction}
\end{equation}
To find the optimal transformation matrix $W_t \in \mathbb{R}^{D_{\text{rep}} \times D_{\text{rep}}}$, we use a set of $N$ calibration samples. We collect their representations into matrices $\mathcal{Z}^\text{mtl}_{t}, \mathcal{Z}^\text{ind}_{t} \in \mathbb{R}^{D_{\text{rep}} \times N}$ and minimize the squared Frobenius norm between the corrected and ideal representations:
\begin{equation}
\min_{W_t} \|W_t \mathcal{Z}^\text{mtl}_{t} - \mathcal{Z}^\text{ind}_{t}\|_F^2
\label{eq:objective_simple}
\end{equation}

\subsection{Optimal Orthogonal Regularization}
Directly solving Eq.~\ref{eq:objective_simple} via ordinary least squares can overfit when the correction data, available at test time, is scarce. To improve robustness, we regularize the solution towards the optimal orthogonal transformation $W_t^{\text{orth}}$, which is solvable via the Orthogonal Procrustes problem (by finding the SVD of the cross-covariance matrix $\mathcal{Z}^\text{ind}_t {\mathcal{Z}^\text{mtl}_t}^\mathsf{T} = U_t S_t V_t^\mathsf{T}$ and setting $W_t^{\text{orth}} = U_t  V_t^\mathsf{T}$).
This structural prior prevents $W_t$ from distorting the geometric structure of the representation space. The regularized objective is:
\begin{equation}
\min_{W_t} \|W_t  \mathcal{Z}^\text{mtl}_{t} - \mathcal{Z}^\text{ind}_{t}\|_F^2 + \beta \|W_t - W_t^{\text{orth}}\|_F^2
\label{eq:orth_reg}
\end{equation}
where $\beta$ is a hyperparameter. This objective has a closed-form solution (see Appendix B.1 for derivation):
\begin{equation}
\hat{W}_t = (\mathcal{Z}^\text{ind}_{t} {\mathcal{Z}^\text{mtl}_{t}}^\mathsf{T} + \beta W_t^{\text{orth}})(\mathcal{Z}^\text{mtl}_{t} {\mathcal{Z}^\text{mtl}_{t}}^\mathsf{T} + \beta I)^{-1}
\label{eq:single_close_form}
\end{equation}
This provides the optimal correction for a single task.

\subsection{Pareto-Optimal Representation Correction}
To handle multiple tasks and user preferences $\mathbf{p}=[p_1,\cdots,p_T]^\mathsf{T}$, we extend the single-task objective (Eq. \ref{eq:orth_reg}) to a multi-objective optimization (MOO) problem:
\begin{equation}
\min_{W} \{L_1(W), \cdots, L_T(W)\},
\label{eq:moo}
\end{equation}
where each $L_t(W)$ is the loss from Eq. \ref{eq:orth_reg}.
We apply Linear Scalarization to find Pareto-optimal solutions, resulting in a single preference-weighted objective:
\begin{equation}
\min_{W} \sum_{t=1}^T p_t L_t(W)
\label{eq:scalarization}
\end{equation}
Crucially, since each loss function $L_t(W)$ is a convex quadratic function of $W$, the weighted objective in Eq.~\ref{eq:scalarization} also possesses this property. This allows us to derive a unique, analytical solution, which we formalize in the following proposition (proof in Appendix B.2):

\begin{proposition}
\label{prop:pareto_optimal_trans}
For any preference $\mathbf{p}$, the unique Pareto-optimal solution $W_\mathbf{p}$ to the multi-objective problem in Eq.~\ref{eq:moo} has a closed-form expression given by:
\begin{equation}
\begin{split}
W_\mathbf{p} = & \left( \sum_{t=1}^{T} p_t (\mathcal{Z}^\text{ind}_t {\mathcal{Z}^\text{mtl}_t}^\mathsf{T} + \beta W^\text{orth}_t) \right) \\ & \left( \sum_{t=1}^{T} p_t (\mathcal{Z}^\text{mtl}_t {\mathcal{Z}^\text{mtl}_t}^\mathsf{T} + \beta I) \right)^{-1}
\end{split}
\label{eq:pareto_cf}
\end{equation}
\end{proposition}

The solution $W_\mathbf{p}$ is the core of our framework, enabling instant generation of a preference-aware model. Moreover, its structure reveals a principled aggregation mechanism. By substituting the single-task optimal solution $\hat{W}_t$ from Eq.~\ref{eq:single_close_form} back into Eq.~\ref{eq:pareto_cf}, we can algebraically rearrange $W_\mathbf{p}$ into the following intuitive form:
\begin{equation}
W_\mathbf{p} = \left( \sum_{t=1}^{T} p_t \hat{W}_t C_t \right) \left( \sum_{t=1}^{T} p_t C_t \right)^{-1}
\label{eq:modular_pareto}
\end{equation}
where $C_t = \mathcal{Z}^\text{mtl}_t {\mathcal{Z}^\text{mtl}_t}^\mathsf{T} + \beta I$. This form clearly shows that our solution is a weighted average of the individual correction maps $\hat{W}_t$, where each map is weighted not just by the user preference $p_t$, but by a data-dependent matrix $C_t$. This is fundamentally different from a naive average, $W_{\text{naive}} = \sum p_t \hat{W}_t$, to which our solution only simplifies if all $C_t$ are identical. Since $C_t$ is the regularized autocorrelation matrix of the features, our method naturally gives greater influence to corrections for tasks with more pronounced, high-variance feature structures. We will validate the empirical superiority of this data-aware weighting in our ablation study. 
Pseudocode details our method's efficient implementation: Algorithm~\ref{alg:porc_offline} covers offline computation of reusable components, while Algorithm~\ref{alg:porc_online} outlines the online two-stage workflow for instant corrector $W_\mathbf{p}$ assembly and inference.

\subsection{Complexity and Scalability}
The analytical, non-iterative nature of our framework leads to its exceptional computational efficiency. The one-time setup for $T$ tasks is dominated by a highly parallelizable representation extraction, followed by matrix computations of $O(T D_{\text{rep}}^3)$ on features of dimension $D_{\text{rep}}$. This approach bypasses the prohibitive evaluation costs and large $O(Td)$ storage of MAP~\citep{liMAPLowcomputeModel2024}, where $d$ denotes the full model parameter count, as well as the extensive, less parallelizable iterative training of PM~\citep{chenParetoMerging2025}. Consequently, our method requires only $O(d + T D_{\text{rep}}^2)$ storage, while adapting to a new preference $\mathbf{p}$ is nearly instantaneous at $O(T D_{\text{rep}}^2 + D_{\text{rep}}^3)$. Since $D_{\text{rep}}$ is a small constant defined by the model architecture (e.g., 512 for ViT-B/32), this cubic complexity arises from a single matrix inversion that executes in \textit{milliseconds}. This offers a uniquely practical solution for real-time generation of preference-aligned models. A more detailed complexity analysis is provided in the Appendix C.

\begin{algorithm}[H]
\caption{One-Time Component Computation (Offline)}
\label{alg:porc_offline}
\textbf{Input}: Merged backbone $f_{\text{merge}}$; Task-specific models $\{(f_t, h_t)\}_{t=1}^T$; Calibration data $\{\mathcal{D}^{\text{calib}}_t\}_{t=1}^T$; Hyperparameter $\beta$. \\
\mbox{\textbf{Output}: Pre-computed components $\{\hat{W}_t, C_t\}_{t=1}^T$.}
\vspace{-\baselineskip}
\begin{algorithmic}[1]
\FOR{$t=1, \dots, T$}
    \STATE Extract representations: $\mathcal{Z}^\text{mtl}_{t} \leftarrow f_{\text{merge}}(\mathcal{D}^{\text{calib}}_t)$, $\mathcal{Z}^\text{ind}_{t} \leftarrow f_t(\mathcal{D}^{\text{calib}}_t)$;
    \STATE Compute cross-covariance: $S_t \leftarrow \mathcal{Z}^\text{ind}_t {\mathcal{Z}^\text{mtl}_t}^\mathsf{T}$;
    \STATE Solve Orthogonal Procrustes problem via SVD:
    \STATE \hspace{\algorithmicindent} $U_t, \_, V_t^\mathsf{T} \leftarrow \text{SVD}(S_t)$;
    \STATE \hspace{\algorithmicindent} $W_t^{\text{orth}} \leftarrow U_t V_t^\mathsf{T}$;
    \STATE Pre-compute components based on Eq.~\ref{eq:modular_pareto}:
    \STATE \hspace{\algorithmicindent} $C_t \leftarrow \mathcal{Z}^\text{mtl}_t {\mathcal{Z}^\text{mtl}_t}^\mathsf{T} + \beta I$;
    \STATE \hspace{\algorithmicindent} $\hat{W}_t \leftarrow (S_t + \beta W_t^{\text{orth}})C_t^{-1}$; 
\ENDFOR
\STATE \textbf{return} $\{\hat{W}_t, C_t\}_{t=1}^T$
\end{algorithmic}
\end{algorithm}

\begin{algorithm}[H]
\caption{Online Preference Adaptation and Inference}
\label{alg:porc_online}
\textbf{Requires}: Pre-computed components $\{\hat{W}_t, C_t\}_{t=1}^T$ from Alg.~\ref{alg:porc_offline}; Merged backbone $f_{\text{merge}}$; A data point $x$ for a target task $t$ with head $h_t$.
\begin{algorithmic}[1]
\STATE \textbf{\textit{\# Step 1: Assemble corrector for preference $\mathbf{p}$ (run once per preference)}} 
\STATE $W_\mathbf{p} \leftarrow \left( \sum_{i=1}^{T} p_i \hat{W}_i C_i \right) \left( \sum_{i=1}^{T} p_i C_i \right)^{-1}$;
\STATE \textbf{return} $W_\mathbf{p}$
\STATE \textbf{\textit{\# Step 2: Use assembled corrector $W_\mathbf{p}$ for inference (run per data point)}}
\STATE Extract representation: $z^{\text{mtl}} \leftarrow f_{\text{merge}}(x)$;
\STATE Apply correction: $\hat{z}^{\text{mtl}} \leftarrow W_\mathbf{p} z^{\text{mtl}}$;
\STATE Predict using task head: $y_t \leftarrow h_t(\hat{z}^{\text{mtl}})$;
\STATE \textbf{return} $y_t$
\end{algorithmic}
\end{algorithm}

\begin{table*}[t!]
\centering
\tabcolsep=0.24em
\begin{tabular*}{\textwidth}{@{\extracolsep{\fill}}c|l|*{8}{c}|c} \toprule
Pref. & Method & SUN397 & Cars & RESISC45 & EuroSAT & SVHN & GTSRB & MNIST & DTD & Average \\
\midrule
\multirow{2}{*}{--} & Individual & 75.3 & 77.7 & 96.1 & 99.7 & 97.5 & 98.7 & 99.7 & 79.4 & 90.5 \\
    & MTL & 73.9 & 74.4 & 93.9 & 98.2 & 95.8 & 98.9 & 99.5 & 77.9 & 88.9 \\
\midrule
\multirow{5}{*}{--} & TA~\citep{DBLP:conf/iclr/IlharcoRWSHF23} & 55.2 & 54.9 & 66.7 & 78.9 & 80.2 & 69.7 & 97.3 & 50.4 & 69.1 \\
    & TIES~\citep{DBLP:conf/nips/YadavTCRB23} & 59.5 & 60.0 & 71.7 & 78.2 & 86.3 & 72.9 & 98.2 & 52.8 & 72.4 \\
    & AM~\citep{DBLP:conf/iclr/YangW00G0T24} & 64.5 & 68.1 & 79.2 & 93.8 & 87.0 & 91.9 & 97.5 & 59.1 & 80.1 \\
    & AMPP~\citep{DBLP:conf/iclr/YangW00G0T24} & 66.6 & 68.3 & 82.2 & 94.2 & 89.6 & 89.0 & 98.3 & 60.6 & 81.1 \\
\midrule
\multirow{7}{*}{\rotatebox{90}{\textbf{equal}}} & MAP~\citep{liMAPLowcomputeModel2024} & 60.0 & 58.8 & 85.8 & 69.5 & 83.5 & 73.4 & 87.8 & 53.2 & 71.5 \\
   & AM+PM~\citep{chenParetoMerging2025}  & 70.1 & \textbf{74.2} & 87.3 & \textbf{96.5} & 90.2 & 95.6 & 98.5 & 66.7 & 84.9 \\
& \cellcolor{gray!20}AM+Ours& \cellcolor{gray!20}71.0& \cellcolor{gray!20}70.0& \cellcolor{gray!20}88.2& \cellcolor{gray!20}95.4& \cellcolor{gray!20}90.9& \cellcolor{gray!20}97.1& \cellcolor{gray!20}98.6& \cellcolor{gray!20}68.0& \cellcolor{gray!20}84.9\\
   & \cellcolor{gray!20}AM+Ours~($10\%$ unlabeled test data)& \cellcolor{gray!20}70.0& \cellcolor{gray!20}69.4& \cellcolor{gray!20}87.6& \cellcolor{gray!20}94.6& \cellcolor{gray!20}90.7& \cellcolor{gray!20}96.6& \cellcolor{gray!20}98.6& \cellcolor{gray!20}66.2& \cellcolor{gray!20}84.2\\
   & AMPP+PM~\citep{chenParetoMerging2025} & 70.6 & 73.9 & 87.5 & 96.7 & 90.8 & 96.7 & 98.6 & 67.2 & 85.2 \\
   & \cellcolor{gray!20}AMPP+Ours& \cellcolor{gray!20}\textbf{72.0}& \cellcolor{gray!20}70.4& \cellcolor{gray!20}\textbf{88.5}& \cellcolor{gray!20}95.8& \cellcolor{gray!20}\textbf{91.7}& \cellcolor{gray!20}\textbf{97.3}& \cellcolor{gray!20}\textbf{98.6}& \cellcolor{gray!20}\textbf{68.5}& \cellcolor{gray!20}\textbf{85.4}\\
   & \cellcolor{gray!20}AMPP+Ours~($10\%$ unlabeled test data)& \cellcolor{gray!20}71.0& \cellcolor{gray!20}69.3& \cellcolor{gray!20}88.0& \cellcolor{gray!20}95.2& \cellcolor{gray!20}91.4& \cellcolor{gray!20}96.8& \cellcolor{gray!20}98.6& \cellcolor{gray!20}67.1& \cellcolor{gray!20}84.7\\
\midrule
\multirow{6}{*}{\rotatebox{90}{\textbf{priority}}} & AM+PM~\citep{chenParetoMerging2025} & 71.1 & \textbf{74.2} & 89.0 & \textbf{97.6} & 92.1 & 97.4 & \textbf{99.0} & 64.0 & 85.5 \\
& \cellcolor{gray!20}AM+Ours& \cellcolor{gray!20}72.5& \cellcolor{gray!20}72.8& \cellcolor{gray!20}91.2& \cellcolor{gray!20}97.1& \cellcolor{gray!20}92.3& \cellcolor{gray!20}98.1& \cellcolor{gray!20}98.8& \cellcolor{gray!20}73.7& \cellcolor{gray!20}87.1\\
& \cellcolor{gray!20}AM+Ours~($10\%$ unlabeled test data)& \cellcolor{gray!20}71.5& \cellcolor{gray!20}71.5& \cellcolor{gray!20}90.4& \cellcolor{gray!20}96.3& \cellcolor{gray!20}92.0& \cellcolor{gray!20}97.4& \cellcolor{gray!20}98.8& \cellcolor{gray!20}71.1& \cellcolor{gray!20}86.1\\
& AMPP+PM~\citep{chenParetoMerging2025}  & 72.1 & 73.7 & 88.8 & 97.5 & 92.2 & 97.5 & \textbf{99.0} & 66.1 & 85.9 \\
& \cellcolor{gray!20}AMPP+Ours& \cellcolor{gray!20}\textbf{73.1}& \cellcolor{gray!20}73.0& \cellcolor{gray!20}\textbf{91.4}& \cellcolor{gray!20}97.5& \cellcolor{gray!20}\textbf{92.7}& \cellcolor{gray!20}\textbf{98.2}& \cellcolor{gray!20}98.8& \cellcolor{gray!20}\textbf{74.0}& \cellcolor{gray!20}\textbf{87.3}\\
& \cellcolor{gray!20}AMPP+Ours~($10\%$ unlabeled test data)& \cellcolor{gray!20}72.3& \cellcolor{gray!20}71.8& \cellcolor{gray!20}90.4& \cellcolor{gray!20}96.8& \cellcolor{gray!20}92.5& \cellcolor{gray!20}97.6& \cellcolor{gray!20}98.8& \cellcolor{gray!20}71.0& \cellcolor{gray!20}86.4\\
\midrule
\multirow{6}{*}{\rotatebox{90}{\textbf{one-hot}}}
& AM+PM~\citep{chenParetoMerging2025} & 70.9 & 59.9 & 89.5 & 96.9 & 85.4 & 97.2 & 99.1 & 62.5 & 82.7 \\
& \cellcolor{gray!20}AM+Ours& \cellcolor{gray!20}72.8& \cellcolor{gray!20}74.7& \cellcolor{gray!20}93.1& \cellcolor{gray!20}98.8& \cellcolor{gray!20}93.4& \cellcolor{gray!20}98.6& \cellcolor{gray!20}\textbf{99.3}& \cellcolor{gray!20}78.2& \cellcolor{gray!20}88.6\\
& \cellcolor{gray!20}AM+Ours~($10\%$ unlabeled test data)& \cellcolor{gray!20}72.0& \cellcolor{gray!20}72.1& \cellcolor{gray!20}91.1& \cellcolor{gray!20}96.3& \cellcolor{gray!20}92.8& \cellcolor{gray!20}97.8& \cellcolor{gray!20}99.0& \cellcolor{gray!20}70.6& \cellcolor{gray!20}86.5\\
& AMPP+PM~\citep{chenParetoMerging2025}  & 71.7 & 62.7 & 89.9 & 97.0 & 88.3 & 97.2 & 99.1 & 63.2 & 83.6 \\
& \cellcolor{gray!20}AMPP+Ours& \cellcolor{gray!20}\textbf{73.4}& \cellcolor{gray!20}\textbf{75.3}& \cellcolor{gray!20}\textbf{93.3}& \cellcolor{gray!20}\textbf{99.0}& \cellcolor{gray!20}\textbf{93.6}& \cellcolor{gray!20}\textbf{98.7}& \cellcolor{gray!20}\textbf{99.3}& \cellcolor{gray!20}\textbf{78.3}& \cellcolor{gray!20}\textbf{88.9}\\
& \cellcolor{gray!20}AMPP+Ours~($10\%$ unlabeled test data)& \cellcolor{gray!20}72.7& \cellcolor{gray!20}72.2& \cellcolor{gray!20}91.3& \cellcolor{gray!20}97.0& \cellcolor{gray!20}93.1& \cellcolor{gray!20}97.9& \cellcolor{gray!20}99.1& \cellcolor{gray!20}71.5& \cellcolor{gray!20}86.7\\
\bottomrule
\end{tabular*}
\caption{Test accuracies (\%) on eight datasets when merging eight ViT-B/32 models. We compare our method against non-controllable baselines and state-of-the-art controllable approaches under three preference scenarios. Our method consistently achieves the highest average accuracy in each group, with the best results highlighted in \textbf{bold}. }
\label{tab:main_comparison}
\end{table*}

\section{Experiments}
Our experiments are designed to answer four key questions:
\begin{itemize}
    \item[\textbf{Q1}] Does our method achieve state-of-the-art performance and controllability?

    \item[\textbf{Q2}] Does it offer a more efficient and scalable paradigm?
    
    \item[\textbf{Q3}] Is its empirical success rooted in its core design principles?

    \item[\textbf{Q4}] How sufficient is our linear correction compared to a non-linear alternative?
\end{itemize}
\subsection{Experimental Setting}
\paragraph{Experiment setup.} Following the protocols of Task Arithmetic~\citep{DBLP:conf/iclr/IlharcoRWSHF23} and Pareto Merging~\citep{chenParetoMerging2025}, we utilize the publicly available models provided by Task Arithmetic. This includes their pre-trained CLIP~\citep{DBLP:conf/icml/RadfordKHRGASAM21} ViT-B/32 visual encoder and the eight individual models already fine-tuned on diverse image classification tasks: SUN397~\citep{xiaoSUNDatabaseLargescale2010}, Cars~\citep{krause3DObjectRepresentations2013}, RESISC45~\citep{chengRemoteSensingImage2017}, EuroSAT~\citep{helberEuroSATNovelDataset2019}, SVHN~\citep{netzerReadingDigitsNatural2011}, GTSRB~\citep{stallkampManVsComputer2012}, MNIST~\citep{lecunGradientbasedLearningApplied1998}, and DTD~\citep{cimpoiDescribingTexturesWild2014}. We position our method against a hierarchy of baselines: performance upper bounds (individually fine-tuned models~(Individual), Multi-Task Learning~(MTL)), non-controllable merging methods (Task Arithmetic~(TA), TIES-Merging~(TIES)~\citep{DBLP:conf/nips/YadavTCRB23}, AdaMerging~(AM), AdaMerging++~(AMPP)~\citep{DBLP:conf/iclr/YangW00G0T24}), and the direct state-of-the-art in controllable merging, Pareto Merging (PM). For fair comparison, all results for our method are generated using the AM backbone unless specified otherwise. Please see Appendix A.1 and A.2 for the introduction of the tasks and implementation details.

\paragraph{Evaluation metrics.}
Our evaluation relies on three metrics. The primary metric for task-specific performance is classification accuracy. To assess the overall quality of the solution set approximating the Pareto front, we employ Hypervolume (HV)~\cite{zitzler1999multiobjective}. To measure how precisely a model's performance aligns with a given user preference, we use Uniformity~\cite{mahapatra2020multi}. To ensure a fair comparison across tasks with disparate difficulty, both metrics are computed on accuracies normalized relative to their individual model's performance. Detailed formulations for HV and Uniformity are provided in the Appendix A.3.

\subsection{Main Results}

To answer \textbf{Q1}, we first show in Table~\ref{tab:main_comparison} that our method provides a substantial uplift over the non-controllable backbones it enhances, such as AMPP (85.4\% vs. 81.1\%). We then compare against leading controllable methods. As evaluating all possible user preferences is infeasible, we test controllability under three representative scenarios reflecting common use cases: equal preference (uniform weights), priority preference (a single task is given 50\% weight, with the remaining weight distributed evenly among the other tasks), and one-hot preference (a single task receives all weight). Following the protocol in~\cite{chenParetoMerging2025}, we report MAP's result only under equal preference. Our primary comparison is against Pareto Merging (PM), focusing on its strongest data-based variants (AM+PM and AMPP+PM). This provides the most direct comparison, as our method is also a data-driven correction. However, our approach acts as a lightweight post-hoc step, in contrast to PM's complex joint optimization with these backbones. The results reveal a clear trend. Our method is competitive with PM in the balanced setting, but its advantage grows as preferences become more focused, culminating in the one-hot scenario. Here our method achieves a striking +5.3\% average accuracy gain over PM (88.9\% vs. 83.6\%), showing a superior ability to satisfy specific targets without the severe performance degradation of prior work. Notably, this performance is so data-efficient that our method using just 10\% of the calibration data still outperforms the fully-resourced PM baseline in the crucial priority and one-hot scenarios.

\begin{figure*}[t!]
    \centering
    \includegraphics[width=0.24\textwidth]{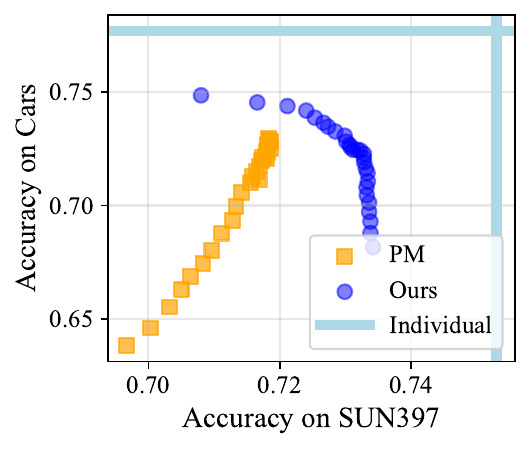}
\includegraphics[width=0.24\textwidth]{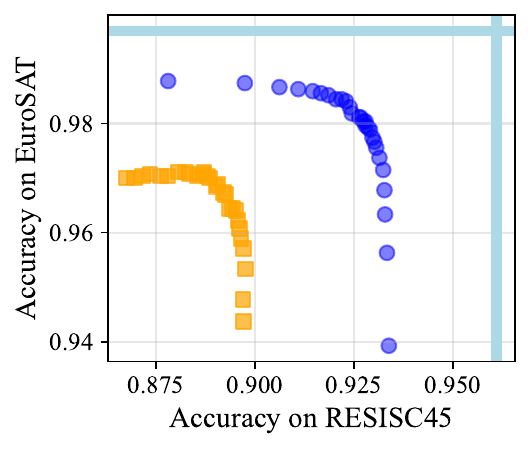}
    \includegraphics[width=0.24\textwidth]{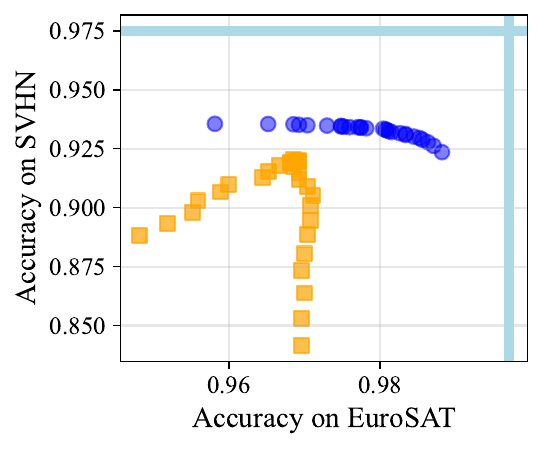}
\includegraphics[width=0.24\textwidth]{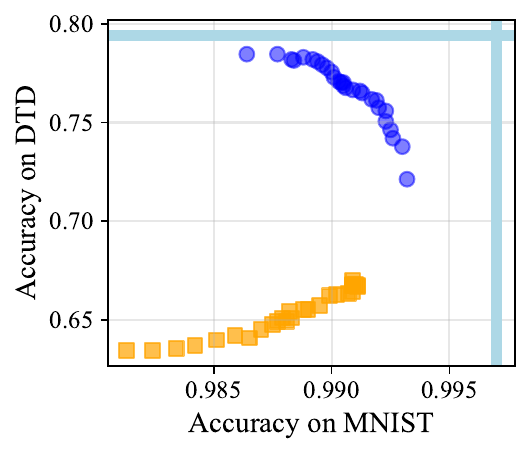}
\caption{Pairwise performance trade-offs within an 8-task merge (AMPP backbone). Each subplot shows the accuracy on a task pair as preferences are varied between them. Our method (blue) consistently achieves a superior and more stable Pareto front compared to Pareto Merging (orange), which fails to produce controllable responses on several critical task pairs (e.g., SUN397-Cars, MNIST-DTD).}
    \label{fig:2task_of_8task}
\end{figure*}

To specifically test controllability against PM, we designed a challenging experiment within the 8-task merge. We focused on a 3-task subspace (Cars, RESISC45, DTD) and allocated 60\% of the total preference weight to them while sampling preferences across their simplex uniformly, and distributing the rest equally among the other five tasks. 
Table~\ref{tab:hv_u} reports the HV and Uniformity (U) from this experiment, measured on both the 3-task subspace (@3) and the full 8-task space (@8). Our method achieves substantially higher scores across all metrics in this challenging scenario. The strong @8 scores show this precision is achieved without disproportionately sacrificing the non-prioritized tasks, demonstrating a more robust and holistic trade-off capability. The high @3 scores confirm its precise control, and Figure~\ref{fig:3task} provides the visual evidence: for our method (b), the accuracy of each task peaks sharply at its corresponding preference corner (e.g., Cars accuracy is highest at corner C). In contrast, PM's response (a) is more diffuse, with accuracy peaks that are less pronounced and often misaligned with the intended preference. 

To further probe this fine-grained control, we examine the direct trade-offs between pairs of tasks. We generated performance curves by smoothly varying the preference weight between two tasks (from 0 to 1), while setting the weights for all other six tasks to zero. Figure~\ref{fig:2task_of_8task} visualizes these pairwise Pareto frontiers. The results highlight two critical advantages of our method. First, our approach consistently produces superior frontiers, offering a better accuracy trade-off than PM. Second, and more importantly, our method demonstrates robust controllability where PM fails. While our method produces smooth, predictable trade-off curves across all pairs, PM's performance collapses on several critical pairings (e.g., SUN397-Cars, MNIST-DTD), failing to generate a controllable response. This shows our method's significantly enhanced reliability in satisfying specific user preferences.

\begin{table}[t!]
\centering
\small \begin{tabular}{l|cc|cc}
\toprule
\textbf{Method} &  \textbf{HV@3 $\uparrow$} & \textbf{HV@8 $\uparrow$} & \textbf{U@3 $\uparrow$} & \textbf{U@8 $\uparrow$} \\
\midrule
AM+PM & 76.77 & 60.90 & 53.47 & 31.98 \\
AM+Ours & \textbf{83.95} & \textbf{68.49} & \textbf{63.30} & \textbf{44.77} \\
\midrule
AMPP+PM & 74.92 & 61.18 & 53.41 & 26.59 \\
AMPP+Ours & \textbf{83.94} & \textbf{70.33} & \textbf{63.13} & \textbf{41.58} \\
\bottomrule
\end{tabular}
\caption{Performance comparison against Pareto Merging (PM). We report the Hypervolume (HV) and Uniformity (U) for both 3-task (@3) and 8-task (@8) settings.}
\label{tab:hv_u}
\end{table}

\begin{figure}[t!]
    \centering
    \begin{subfigure}[t]{0.48\textwidth}
     \includegraphics[width=\textwidth]{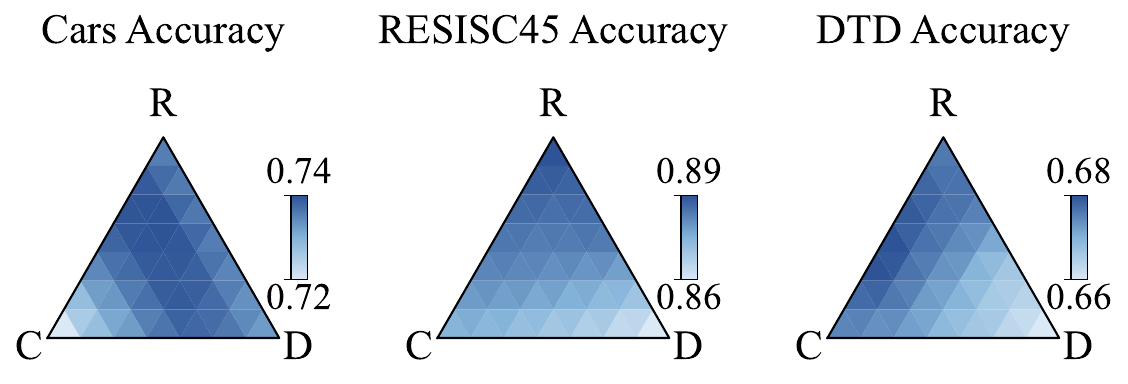}
    \caption{PM}
    \hfill
    \end{subfigure}
    \begin{subfigure}[t]{0.48\textwidth}
         \includegraphics[width=\textwidth]{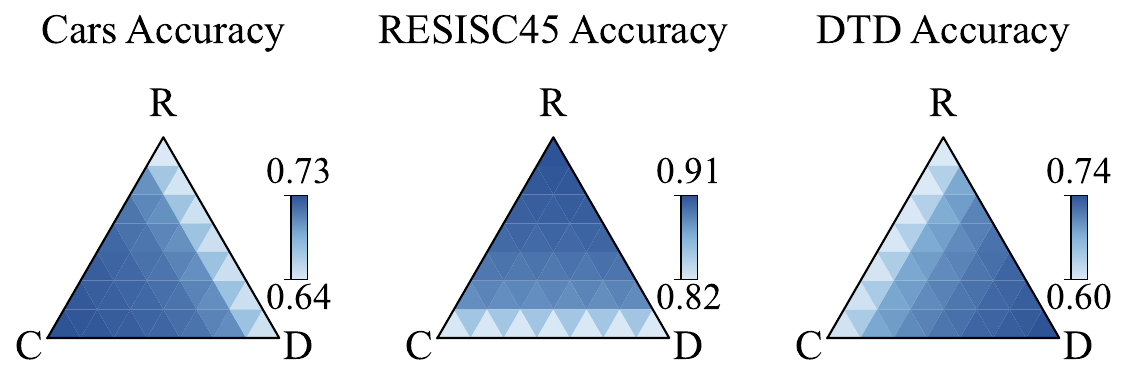}
    \caption{Ours}
    \end{subfigure}
    \caption{Visual evidence for the superior U@3 metrics in Table~\ref{tab:hv_u}. On a 3-task (C: Cars, R: RESISC45, D: DTD) subspace of an 8-task merge, our method (b) shows sharp, predictable accuracy peaks. This ideal control landscape directly explains its higher Uniformity. In contrast, PM (a) yields a misaligned response, resulting in lower scores.}
    \label{fig:3task}
\end{figure}

\begin{figure}[t!]
    \centering
{\includegraphics[width=0.48\textwidth]{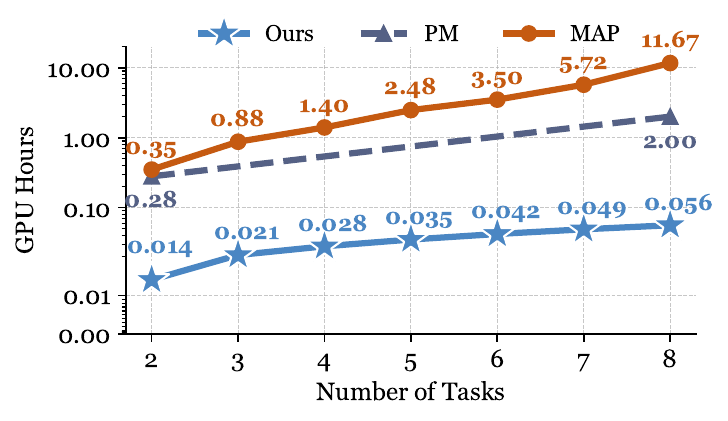}}
    \caption{Scalability comparison. We plot the time for the preference-aware stage against the number of tasks. Note the logarithmic scale on the y-axis, which highlights the orders-of-magnitude difference in efficiency.}
    \label{fig:scalability}
\end{figure}

\subsection{Efficiency and Scalability Analysis}
To answer \textbf{Q2}, we analyze our method's efficiency and scalability, focusing on the critical cost of enabling preference-awareness. Our method's efficiency stems from its decoupled design, which operates on a pre-merged backbone in a distinct, subsequent step. This contrasts with approaches like PM that require a costly joint optimization of their preference-aware mechanism alongside the backbone merge itself. Figure~\ref{fig:scalability} plots the computation time for the generation of the preference-aware model (for MAP) or the auxiliary module (for PM and Ours), confirming our method's linear time complexity. For an 8-task merge, our method completes in just 0.056 GPU hours, achieving an approximately 36x speedup over PM and a 208x speedup over MAP. For PM, we plot its reported costs for 2 and 8 tasks—the only configurations detailed in their work—and connect them with a dashed line to illustrate the scaling trend of its iterative optimization framework. This efficiency advantage is critical for practical applications and enables on-the-fly generation of tailored models.

\subsection{Ablation and Mechanism Validation}
To answer \textbf{Q3}, we analyze our framework's key components and hyperparameters.

\paragraph{Efficacy of optimal orthogonal regularization.}
We first study the effect of the regularization strength, $\beta$, which balances data-fitting with structural preservation. As Figure \ref{fig:regularizaton} shows, ablating the regularization term ($\beta=0$) leaves only the least-squares objective. This leads to overfitting on the calibration data, causing poor performance, especially when data is scarce. Introducing the regularization ($\beta>0$) provides a substantial performance boost by pulling the solution towards an optimal orthogonal transformation, which prevents the representation space from distorting.
The method is also robust to the specific choice of $\beta$, as performance remains stable across a wide range of values. The regularization enables high data efficiency. Even with only 10\% of the test data, our method achieves near-optimal accuracy. Finally, the regularized solution slightly outperforms a pure orthogonal transformation ($\beta \to \infty$, noted as \textbf{Orth.}), confirming that our method effectively combines a strong structural prior with valuable data-driven correction.

\begin{figure}[t!]
    \centering
    \clipbox{0.1cm 0.0cm 0.0cm 0.0cm}
{\includegraphics[width=0.48\textwidth]{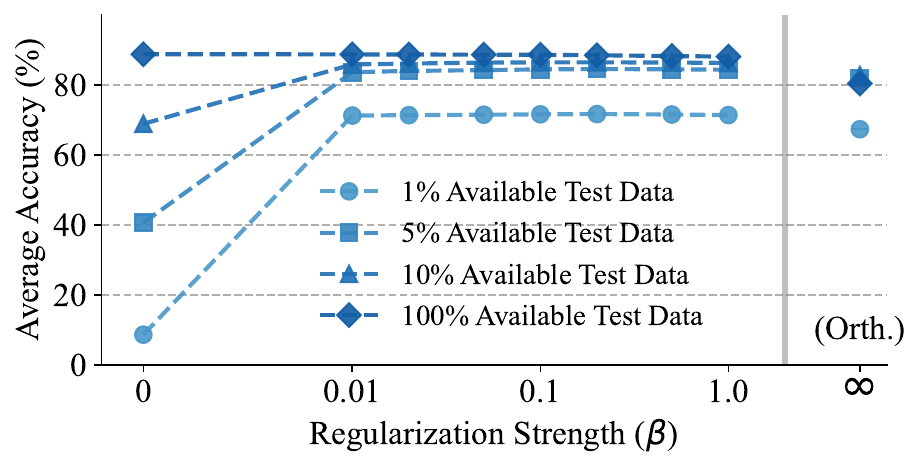}}
    \caption{Accuracy vs. regularization strength $\beta$. Our method is robust to the choice of $\beta$ and highly data-efficient, achieving near-optimal performance with only 10\% of test data.}
    \label{fig:regularizaton}
\end{figure}

\paragraph{Aggregation strategies.}
We compare our Pareto-optimal aggregation against two intuitive yet sub-optimal baselines for preference-aware weighting. The first~(Naive) simply applies a weighted average to the optimal linear transformation of each task. The second~(Polar) uses polar decomposition to separately weight the rotation and scaling components of each transformation before recombination. As shown in Table~\ref{tab:weighting_comparison}, while simpler heuristics can narrowly focus on the priority task, our principled aggregation achieves a much better overall trade-off. This result empirically validates our analysis: the data-aware weighting matrices $C_t$, which account for the feature structure of each task, are critical for finding a globally optimal solution, unlike Naive and Polar strategies.

\begin{table}[t!]
\centering
\small
\setlength{\tabcolsep}{3.5pt} \renewcommand{\arraystretch}{1.1} \begin{tabular}{@{}l|c|ccc@{}} \toprule
\multirow{2}{*}{\textbf{Method}} & 
    \textbf{Equal Pref.} & \multicolumn{3}{c}{\textbf{Priority Pref.}} \\ 
\cmidrule(lr){3-5} 
& \textbf{NAcc.}$\uparrow$ & \textbf{Pri. NAcc.}$\uparrow$ & \textbf{Non-Pri. NAcc.}$\uparrow$ & \textbf{HV}$\uparrow$ \\ 
        
\midrule
Naive & 92.2          & 97.1 & 87.7 & 67.6 \\
Polar   & 91.5          & \textbf{97.2} & 84.9 & 64.0 \\
Ours    & \textbf{93.5} & 96.0 & \textbf{92.6} & \textbf{70.2} \\
\bottomrule
\end{tabular}
\caption{Comparison of aggregation strategies under equal and priority preference settings. We report Normalized Accuracy (NAcc.), Priority-task NAcc., Non-Priority NAcc., and Hypervolume (HV).}
\label{tab:weighting_comparison}
\end{table}

\begin{figure}[t!]
    \centering
{\includegraphics[width=0.45\textwidth]{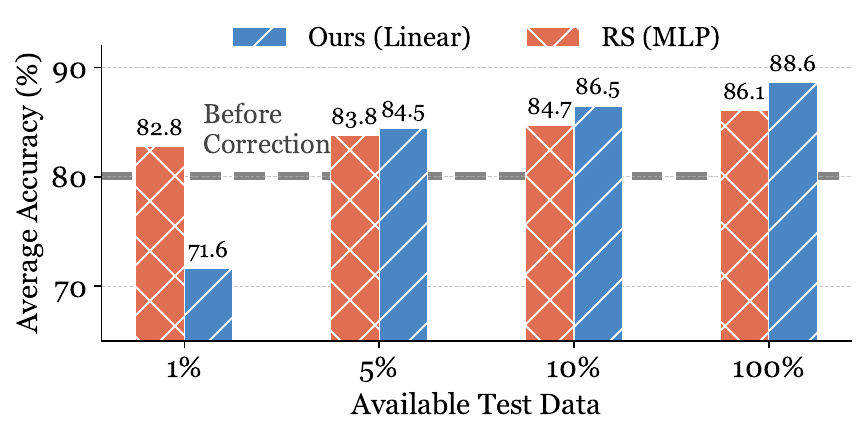}}
    \caption{Linear vs. non-linear correction. Comparison of our global linear map against the non-linear MLP corrector from Representation Surgery (RS) across varying amounts of available test data. The dashed line indicates the accuracy of the merged model (AM) before any correction is applied.}
    \label{fig:correction_comp}
\end{figure}
\subsection{Sufficiency of Linear Correction}

To answer \textbf{Q4}, we compare our linear correction to the non-linear MLP-based from Representation Surgery (RS). As shown in Figure~\ref{fig:correction_comp}, with minimal data (1\%), RS appears more robust, as our linear solution is sensitive to noise in this low-data regime. However, our method's performance rapidly surpasses the RS with just 5\% of the data and widens its lead thereafter. This demonstrates that our linear approach is more data-efficient, better capturing the underlying global distortion once a minimum viable sample set is available.

\section{Conclusion}  
We introduced ReACT, a new \textit{on-the-fly} framework for controllable model merging that bypasses costly parameter-space optimization by directly correcting the model's representations. Our key insight is that this correction can be modeled as a simple, regularized linear map, for which we derived a closed-form, Pareto-optimal solution in the representation alignment space. This analytical approach leads to a highly efficient and scalable method. Experiments confirmed ReACT achieves a state-of-the-art Pareto front and superior preference alignment in terms of task performance.
 
However, our method has limitations: the linear scalarization may not cover the entirety of a concave Pareto front in the performance space, and it assumes predominantly linear distortions requiring a small unlabeled calibration dataset. We hope our work encourages further exploration of simple, analytical solutions in representation space for complex model combination problems, including extending to models with more complex representation structures and exploring calibration-free alternatives.

\section*{Acknowledgements}
This work was supported by the National Natural Science Foundation of China under Grant 62071481.
 
\bibliography{references}

\clearpage

\appendix
\setcounter{secnumdepth}{3}
\twocolumn[
\begin{center}
    \vspace*{0.5em} 
    {\LARGE\bfseries \thetitle\par} 
    \vspace{1.2em}
    {\LARGE\bfseries Appendix\par} 
    
    \vspace{2.5em}
\end{center}
] 

\setcounter{section}{0}
\setcounter{theorem}{0}
\setcounter{figure}{0}
\setcounter{table}{0}
\setcounter{equation}{0}
\renewcommand{\thetheorem}{A.\arabic{theorem}}
\renewcommand{\theequation}{A.\arabic{equation}}
\renewcommand{\thefigure}{A.\arabic{figure}}
\renewcommand{\thetable}{A.\arabic{table}}

\section{Experimental Details}
\label{sec:appendix_exp_details}

\begin{table*}[t!]
\centering
\begin{tabular}{llccr}
\toprule
\textbf{Dataset} & \textbf{Domain} & \textbf{Image Size} & \textbf{\# Classes} & \textbf{\# Train / Test Samples} \\
\midrule
SUN397 \citep{xiaoSUNDatabaseLargescale2010} & Scene classification & Varies & 397 & 17,865 / 19,850 \\
Cars \citep{krause3DObjectRepresentations2013} & Car classification & Varies & 196 & 7,330 / 8,041 \\
RESISC45 \citep{chengRemoteSensingImage2017} & Remote sensing scene & $256 \times 256$ & 45 & 17,010 / 6,300 \\
EuroSAT \citep{helberEuroSATNovelDataset2019} & Satellite image & $64 \times 64$ & 10 & 21,600 / 2,700 \\
SVHN \citep{netzerReadingDigitsNatural2011} & House numbers & $32 \times 32$ & 10 & 68,257 / 26,032 \\
GTSRB \citep{stallkampManVsComputer2012} & Traffic sign & Varies & 43 & 23,976 / 12,630 \\
MNIST \citep{lecunGradientbasedLearningApplied1998} & Handwritten digits & $28 \times 28$ & 10 & 55,000 / 10,000 \\
DTD \citep{cimpoiDescribingTexturesWild2014} & Texture classification & Varies & 47 & 1,692 / 1,880 \\
\bottomrule
\end{tabular}
\caption{Summary of the eight image classification datasets used in our experiments. For each dataset, we list its domain, original image size, number of classes, and the number of samples in the training and testing splits, respectively.}
\label{tab:appendix_dataset_details}
\end{table*}

\subsection{Dataset Details}
Our experiments are conducted on a diverse suite of eight public image classification datasets, following the protocol established by \citet{DBLP:conf/iclr/IlharcoRWSHF23}. These datasets span various domains, including scene recognition, object classification, and digit recognition, presenting a challenging and comprehensive benchmark for evaluating model merging techniques. A detailed summary of the datasets, including the number of classes and the size of training and testing splits used in our evaluation, is provided in Table \ref{tab:appendix_dataset_details}.

\subsection{Implementation and Hyperparameter Settings}
To maintain consistency with prior work and isolate the performance gains from the merging algorithm itself, we use the publicly available fine-tuned models provided by \citet{DBLP:conf/iclr/IlharcoRWSHF23}. These models are based on the CLIP ViT-B/32 visual encoder \citep{DBLP:conf/icml/RadfordKHRGASAM21} pre-trained on ImageNet dataset. 
All experiments were conducted using PyTorch on a server equipped with a NVIDIA RTX 4090 GPU. For merging methods, including Task Arithmetic, TIES-Merging, AdaMerging (AM), AdaMerging++ (AMPP), we utilized FusionBench~\citep{DBLP:journals/corr/abs-2406-03280} implementations, and official implementation for Pareto Merging (PM). All baselines' implementations are followed their recommended hyperparameter settings to ensure fair and robust comparisons. 

Our method operates as a post-hoc correction applied to a pre-merged model. This pre-merged model can be obtained from any existing merging technique; in our main experiments (Table \ref{tab:main_comparison}), we use AdaMerging (AM) or AdaMerging++ (AMPP) as the backbone. The core of our method involves computing the correction matrix $W_\mathbf{p}$ using unlabeled calibration dataset. For all experiments, we construct this calibration set from the official test set of each respective task. The key regularization hyperparameter, $\beta$, was set to $0.1$.
Specifically, the representations $Z_{mtl}$ and $Z_{ind}$ refer to the final feature representations extracted from the visual encoder (e.g., CLIP ViT-B/32 in our case) immediately before being fed into task-specific classification heads. These are the inputs to the task-specific classification heads $h_t$ for prediction, after our proposed linear correction $W_p$ has been applied (i.e., $\hat{z}_{mtl} = W_p z_{mtl}$ is then passed to $h_t(\hat{z}_{mtl})$).

\subsection{Evaluation Metric Formulations}
\label{subsec:appendix_metrics}
As stated in the main paper, we use Hypervolume (HV) and Uniformity (U) to evaluate the quality of the generated Pareto front and the precision of preference alignment, respectively. Both metrics are computed using a normalized performance vector to ensure fair comparison across tasks of varying difficulty.

First, we define the \textbf{normalized accuracy vector}. For a set of $T$ tasks, let $\mathbf{A} = [A_1, A_2, \dots, A_T]^\mathsf{T}$ be the vector of test accuracies achieved by a merged model. Let $\mathbf{A}_{\text{expert}} = [A_{1}^\text{expert}, A_{2}^\text{expert}, \dots, A_{T}^\text{expert}]^\mathsf{T}$ be the vector of accuracies achieved by the individual expert models. The normalized accuracy vector $\mathbf{a}$ is then computed element-wise:
\begin{equation}
    \mathbf{a} = [a_1, a_2, \dots, a_T]^\mathsf{T}, \quad \text{where} \quad a_t = \frac{A_t}{A_{t}^\text{expert}}
\end{equation}
This normalization scales each task's performance relative to its expert's upper-bound, with a value of $1.0$ indicating that the merged model has matched the expert's performance on that task.

\paragraph{Hypervolume (HV).}
The Hypervolume (HV) indicator is a standard metric in multi-objective optimization for quantifying the quality of a set of non-dominated solutions (an approximate Pareto front) \citep{zitzler1999multiobjective}. Given a set of solution vectors $S = \{\mathbf{a}^{(1)}, \mathbf{a}^{(2)}, \dots, \mathbf{a}^{(k)}\}$ in the $T$-dimensional normalized accuracy space, the HV is the volume of the region dominated by these solutions and bounded by a reference point $\mathbf{r}$. Since our objective is to maximize accuracy, a higher HV indicates a better Pareto front.

Formally, the Hypervolume is defined as the Lebesgue measure $\lambda$ of the union of hyperrectangles formed by each solution $\mathbf{a} \in S$ and the reference point $\mathbf{r}$:
\begin{equation}
    \text{HV}(S) = \lambda \left( \bigcup_{\mathbf{a} \in S} \{ \mathbf{z} \in \mathbb{R}^T \mid \mathbf{r} \preceq \mathbf{z} \preceq \mathbf{a} \} \right)
\end{equation}
where $\preceq$ denotes element-wise inequality. In our experiments, we use the origin $\mathbf{r} = [0, 0, \dots, 0]^\mathsf{T}$ as the reference point, which is a standard choice for objectives that are non-negative.

\paragraph{Uniformity (U).}
The Uniformity metric is designed to measure how well the performance of a merged model aligns with a given user preference vector $\mathbf{p} = [p_1, p_2, \dots, p_T]^\mathsf{T}$. We adopt the principle from \citet{mahapatra2020multi}, adapting it from a loss-minimization to an accuracy-maximization context.

The core idea is that for a model to be perfectly aligned with user preferences, a higher preference $p_t$ for a task should result in a smaller performance gap relative to the expert model. We quantify this gap using the \textbf{normalized performance shortfall}, defined for each task as $s_t = 1 - a_t$, where $a_t$ is the normalized accuracy. The vector of shortfalls is $\mathbf{s} = [s_1, s_2, \dots, s_T]^\mathsf{T}$.

Perfect alignment is achieved when the preference-weighted shortfall is constant across all tasks, i.e., $p_1 s_1 = p_2 s_2 = \dots = p_T s_T$. To measure the deviation from this ideal, we first construct a probability distribution $\hat{\mathbf{s}}$ from these weighted shortfalls:
\begin{equation}
    \hat{s}_t = \frac{p_t s_t}{\sum_{i=1}^T p_i s_i}
\end{equation}
In the ideal case, this distribution becomes the uniform distribution $\mathbf{u} = [1/T, 1/T, \dots, 1/T]^\mathsf{T}$. The Uniformity score is then defined by how closely $\hat{\mathbf{s}}$ resembles $\mathbf{u}$, measured via the Kullback-Leibler (KL) divergence:
\begin{equation}
    U(\mathbf{a}, \mathbf{p}) = 1 - D_{KL}(\hat{\mathbf{s}} || \mathbf{u}) = 1 - \sum_{t=1}^T \hat{s}_t \log(T \hat{s}_t)
\end{equation}
A score closer to 1 indicates a smaller KL divergence and thus a more precise adherence of the model's performance to the user-specified priorities.

\section{Derivations and Proofs}
\label{sec:appendix_proofs}

\subsection{Derivation of the Optimal Regularized Linear Map}
\label{subsec:appendix_derivation}

In this section, we provide a detailed derivation for the closed-form solution of the regularized linear correction map $\hat{W}_t$ presented in Eq.~(\ref{eq:single_close_form}) of the main text. Our starting point is the regularized objective function from Eq.~(\ref{eq:orth_reg}), which seeks to find an optimal transformation matrix $W_t$ that minimizes the discrepancy between the corrected and ideal representations, while being regularized towards an orthogonal transformation $W_t^{\text{orth}}$:
\begin{equation}
L(W_t) = \|W_t \mathcal{Z}^\text{mtl}_{t} - \mathcal{Z}^\text{ind}_{t}\|_F^2 + \beta \|W_t - W_t^{\text{orth}}\|_F^2
\label{eq:appendix_objective_single}
\end{equation}
Here, $\mathcal{Z}^\text{mtl}_{t}, \mathcal{Z}^\text{ind}_{t} \in \mathbb{R}^{D_{\text{rep}} \times N}$ are matrices containing the representations for $N$ calibration samples, and $\beta$ is a non-negative regularization hyperparameter.

To find the minimum, we can expand the Frobenius norms using the property $\|A\|_F^2 = \mathrm{tr}(A A^\mathsf{T})$. The objective function $L(W_t)$ can be rewritten as:
\begin{align*}
L(W_t) = & \mathrm{tr}\left( (W_t \mathcal{Z}^\text{mtl}_{t} - \mathcal{Z}^\text{ind}_{t})(W_t \mathcal{Z}^\text{mtl}_{t} - \mathcal{Z}^\text{ind}_{t})^\mathsf{T} \right) \\
& + \beta \cdot \mathrm{tr}\left( (W_t - W_t^{\text{orth}})(W_t - W_t^{\text{orth}})^\mathsf{T} \right)
\end{align*}
Expanding the terms within the trace yields:
\begin{align*}
L(W_t) = & \mathrm{tr}(W_t \mathcal{Z}^\text{mtl}_{t} {\mathcal{Z}^\text{mtl}_{t}}^\mathsf{T} W_t^\mathsf{T} - 2 W_t \mathcal{Z}^\text{mtl}_{t} {\mathcal{Z}^\text{ind}_{t}}^\mathsf{T} + \mathcal{Z}^\text{ind}_{t} {\mathcal{Z}^\text{ind}_{t}}^\mathsf{T}) \\
& + \beta \cdot \mathrm{tr}(W_t W_t^\mathsf{T} - 2 W_t {W_t^{\text{orth}}}^\mathsf{T} + W_t^{\text{orth}} {W_t^{\text{orth}}}^\mathsf{T})
\end{align*}
This objective is a convex quadratic function of $W_t$. We can find the optimal $W_t$ by taking the derivative of $L(W_t)$ with respect to $W_t$ and setting it to zero. Using standard matrix calculus identities, specifically $\frac{\partial}{\partial X} \mathrm{tr}(AXB) = A^\mathsf{T} B^\mathsf{T}$ and $\frac{\partial}{\partial X} \mathrm{tr}(XX^\mathsf{T}) = 2X$, we compute the gradient:
\begin{align*}
\frac{\partial L(W_t)}{\partial W_t} = & \frac{\partial}{\partial W_t} \mathrm{tr}(W_t (\mathcal{Z}^\text{mtl}_{t} {\mathcal{Z}^\text{mtl}_{t}}^\mathsf{T}) W_t^\mathsf{T}) \\
& - 2 \frac{\partial}{\partial W_t} \mathrm{tr}(W_t (\mathcal{Z}^\text{mtl}_{t} {\mathcal{Z}^\text{ind}_{t}}^\mathsf{T})) \\
& + \beta \frac{\partial}{\partial W_t} \mathrm{tr}(W_t W_t^\mathsf{T}) - 2 \beta \frac{\partial}{\partial W_t} \mathrm{tr}(W_t {W_t^{\text{orth}}}^\mathsf{T}) \\
= & 2 W_t (\mathcal{Z}^\text{mtl}_{t} {\mathcal{Z}^\text{mtl}_{t}}^\mathsf{T}) - 2 \mathcal{Z}^\text{ind}_{t} {\mathcal{Z}^\text{mtl}_{t}}^\mathsf{T} \\
& + 2 \beta W_t - 2 \beta W_t^{\text{orth}}
\end{align*}
Setting the gradient to zero gives us the following equation:
\begin{equation*}
W_t (\mathcal{Z}^\text{mtl}_{t} {\mathcal{Z}^\text{mtl}_{t}}^\mathsf{T}) - \mathcal{Z}^\text{ind}_{t} {\mathcal{Z}^\text{mtl}_{t}}^\mathsf{T} + \beta W_t - \beta W_t^{\text{orth}} = 0
\end{equation*}
Now, we can solve for $W_t$ by isolating the terms containing it:
\begin{align*}
W_t (\mathcal{Z}^\text{mtl}_{t} {\mathcal{Z}^\text{mtl}_{t}}^\mathsf{T}) + \beta W_t I &= \mathcal{Z}^\text{ind}_{t} {\mathcal{Z}^\text{mtl}_{t}}^\mathsf{T} + \beta W_t^{\text{orth}} \\
W_t (\mathcal{Z}^\text{mtl}_{t} {\mathcal{Z}^\text{mtl}_{t}}^\mathsf{T} + \beta I) &= \mathcal{Z}^\text{ind}_{t} {\mathcal{Z}^\text{mtl}_{t}}^\mathsf{T} + \beta W_t^{\text{orth}}
\end{align*}
Finally, by right-multiplying with the inverse of the term in the parenthesis, we arrive at the closed-form solution for the optimal regularized linear map, which we denote as $\hat{W}_t$:
\begin{equation}
\hat{W}_t = (\mathcal{Z}^\text{ind}_{t} {\mathcal{Z}^\text{mtl}_{t}}^\mathsf{T} + \beta W_t^{\text{orth}})(\mathcal{Z}^\text{mtl}_{t} {\mathcal{Z}^\text{mtl}_{t}}^\mathsf{T} + \beta I)^{-1}
\end{equation}

\subsection{Proof of Pareto Optimality (Proposition 1)}
\label{subsec:appendix_proof_prop1}

This section provides the proof for Proposition~1 as stated in the main text. We first restate the proposition for clarity.

\begin{proposition}
For any preference $\mathbf{p}$, the unique Pareto-optimal solution $W_\mathbf{p}$ to the multi-objective problem in Eq.~\ref{eq:moo} has a closed-form expression given by:
\begin{equation}
\begin{split}
W_\mathbf{p} = & \left( \sum_{t=1}^{T} p_t (\mathcal{Z}^\text{ind}_t {\mathcal{Z}^\text{mtl}_t}^\mathsf{T} + \beta W^\text{orth}_t) \right) \\ & \left( \sum_{t=1}^{T} p_t (\mathcal{Z}^\text{mtl}_t {\mathcal{Z}^\text{mtl}_t}^\mathsf{T} + \beta I) \right)^{-1}
\end{split}
\end{equation}
\end{proposition}

\begin{proof}
The proof consists of two main parts. First, we establish the Pareto optimality of any solution found via linear scalarization. Second, we derive the specific closed-form expression for this solution.

\paragraph{Part 1: Pareto Optimality.}
The foundation of this proof lies in the convexity of the individual objective functions. For each task $t$, the objective function is given by:
\begin{equation*}
L_t(W) = \|W \mathcal{Z}^\text{mtl}_{t} - \mathcal{Z}^\text{ind}_{t}\|_F^2 + \beta \|W - W_t^{\text{orth}}\|_F^2
\end{equation*}
This function is a sum of two squared Frobenius norms. The term $\|W \mathcal{Z}^\text{mtl}_{t} - \mathcal{Z}^\text{ind}_{t}\|_F^2$ is the squared $\ell_2$-norm of an affine transformation of the elements of $W$, which is convex. Similarly, $\|W - W_t^{\text{orth}}\|_F^2$ is a strictly convex quadratic function of $W$. Since the sum of convex functions is also convex, and $\beta \ge 0$, each individual objective $L_t(W)$ is a convex function of $W$.

The linear scalarization method combines these individual objectives using a preference vector $\mathbf{p} = [p_1, \dots, p_T]^\mathsf{T}$ with $p_t \ge 0$ and $\sum p_t = 1$, creating a single composite objective:
\begin{equation*}
L_{\mathbf{p}}(W) = \sum_{t=1}^{T} p_t L_t(W)
\end{equation*}
As this is a non-negative weighted sum of convex functions, the composite objective $L_{\mathbf{p}}(W)$ is also convex. A fundamental theorem in multi-objective optimization states that for a convex multi-objective problem, any minimizer of the scalarized objective with strictly positive weights ($p_t > 0$ for all $t$) is a Pareto-optimal solution for the original vector optimization problem. Even if some weights are zero, the solution is guaranteed to be weakly Pareto-optimal, and is Pareto-optimal under mild additional conditions that hold in our case. Thus, the solution $W_\mathbf{p}$ that minimizes $L_{\mathbf{p}}(W)$ is indeed Pareto-optimal.

\paragraph{Part 2: Derivation of the Closed-Form Solution.}
To find the unique minimizer $W_\mathbf{p}$ of the convex function $L_{\mathbf{p}}(W)$, we take its gradient with respect to $W$ and set it to zero. Leveraging the linearity of differentiation, we have:
\begin{equation*}
\frac{\partial L_{\mathbf{p}}(W)}{\partial W} = \frac{\partial}{\partial W} \sum_{t=1}^{T} p_t L_t(W) = \sum_{t=1}^{T} p_t \frac{\partial L_t(W)}{\partial W}
\end{equation*}
From the derivation in Appendix~C.\ref{subsec:appendix_derivation}, we already have the expression for the gradient of the individual objective $L_t(W)$:
\begin{equation*}
\frac{\partial L_t(W)}{\partial W} = 2 W (\mathcal{Z}^\text{mtl}_{t} {\mathcal{Z}^\text{mtl}_{t}}^\mathsf{T} + \beta I) - 2 (\mathcal{Z}^\text{ind}_{t} {\mathcal{Z}^\text{mtl}_{t}}^\mathsf{T} + \beta W_t^{\text{orth}})
\end{equation*}
Substituting this into the equation for the composite gradient and setting it to zero yields:
\begin{equation*}
\sum_{t=1}^{T} p_t \left( 2 W (\mathcal{Z}^\text{mtl}_{t} {\mathcal{Z}^\text{mtl}_{t}}^\mathsf{T} + \beta I) - 2 (\mathcal{Z}^\text{ind}_{t} {\mathcal{Z}^\text{mtl}_{t}}^\mathsf{T} + \beta W_t^{\text{orth}}) \right) = 0
\end{equation*}
We can rearrange the terms to solve for $W$. First, we move all terms not involving $W$ to the right-hand side:
\begin{align*}
\sum_{t=1}^{T} p_t W (\mathcal{Z}^\text{mtl}_{t} {\mathcal{Z}^\text{mtl}_{t}}^\mathsf{T} + \beta I) = \sum_{t=1}^{T} p_t (\mathcal{Z}^\text{ind}_{t} {\mathcal{Z}^\text{mtl}_{t}}^\mathsf{T} + \beta W_t^{\text{orth}})
\end{align*}
Next, we factor out $W$ from the left-hand side:
\begin{align*}
W \left( \sum_{t=1}^{T} p_t (\mathcal{Z}^\text{mtl}_{t} {\mathcal{Z}^\text{mtl}_{t}}^\mathsf{T} + \beta I) \right) = \sum_{t=1}^{T} p_t (\mathcal{Z}^\text{ind}_{t} {\mathcal{Z}^\text{mtl}_{t}}^\mathsf{T} + \beta W_t^{\text{orth}})
\end{align*}
Finally, by right-multiplying by the inverse of the matrix factor on the left, we obtain the unique closed-form solution:
\begin{equation}
\begin{split}
W_\mathbf{p} = & \left( \sum_{t=1}^{T} p_t (\mathcal{Z}^\text{ind}_t {\mathcal{Z}^\text{mtl}_t}^\mathsf{T} + \beta W^\text{orth}_t) \right) \\ & \left(\sum_{t=1}^{T} p_t (\mathcal{Z}^\text{mtl}_t {\mathcal{Z}^\text{mtl}_t}^\mathsf{T} + \beta I) \right)^{-1}  
\end{split}
\label{eq:appendix_pareto_cf}
\end{equation}
This concludes the proof of Proposition~1.
\end{proof}

\section{Detailed Complexity Analysis}

\begin{table*}[t!]
\centering
\resizebox{\textwidth}{!}{\begin{tabular}{@{}l|ll|ll@{}}
\toprule
\textbf{Method} & \textbf{One-Time Setup (Time)} & \textbf{Setup (Storage)} & \textbf{On-the-fly Generation (Time)} & \textbf{Final Model (Storage)} \\ \midrule
MAP~\citep{liMAPLowcomputeModel2024} & $O(K T \cdot T_{\text{eval}} + K T^5)$ (very high) & $O(Td)$ (high, scales with T) & $O(Td)$ (slow) & $O(d)$ \\
Pareto Merging~\citep{chenParetoMerging2025} & $O(T_{\text{steps}} T \cdot T_{\text{prop}})$ (very high) & $O(d)$ (lowest) & $O(1)$ (near-instant) & $O(d)$ \\
\textbf{Ours} & $O(T(N D_{\text{rep}}^2 + D_{\text{rep}}^3))$ (low, non-iterative) & $O(d + T D_{\text{rep}}^2)$ (low) & $O(T D_{\text{rep}}^2 + D_{\text{rep}}^3)$ (near-instant) & $O(d + D_{\text{rep}}^2)$ \\ \bottomrule
\end{tabular}}
\caption{Complexity comparison of controllable model merging methods. $T$: number of tasks, $d$: model parameters, $D_{\text{rep}}$: representation dimension, $N$: calibration set size. $K$ is the number of samples for MAP's surrogate model fitting, and $T_{\text{eval}}$, $T_{\text{prop}}$ are costs for model evaluation and propagation, respectively. Our method offers a superior trade-off, avoiding the prohibitive setup costs of PM and MAP and the high storage of MAP.}
\label{tab:complexity_summary}
\end{table*}

In this section, we provide a detailed analysis of the storage and computational complexity of our proposed method. We compare it against two leading controllable merging frameworks: MAP~\citep{liMAPLowcomputeModel2024} and Pareto Merging (PM)~\citep{chenParetoMerging2025}. For clarity, we define the following variables used throughout the analysis: $T$ is the number of tasks to be merged, $d$ is the number of parameters in the backbone model (i.e., the dimensionality of a task vector), $D_{\text{rep}}$ is the dimension of the model's representation space where the correction is applied, and $N$ is the number of samples in the calibration set used to compute the representations.

\subsection{Storage Complexity}
The storage requirements of a controllable merging method determine its practicality, especially as the number of tasks $T$ increases. The methods exhibit significant differences in this regard.

Our method requires storing the base merged model $\theta_{\text{merge}}$ (e.g., from AdaMerging), which has $O(d)$ parameters. Additionally, to enable on-the-fly generation, we must store the pre-computed matrices for each task, as derived in Proposition~\ref{prop:pareto_optimal_trans}. Specifically, for each of the $T$ tasks, we store two matrices of size $D_{\text{rep}} \times D_{\text{rep}}$. This results in a total storage complexity of $O(d + T D_{\text{rep}}^2)$. Since $D_{\text{rep}} \ll d$, this is highly efficient.

In contrast, MAP requires storing the full task vector $\Delta_t$ for every task, in addition to the base model $\theta_0$. This is because its inference-time step involves a linear combination of these high-dimensional vectors. Consequently, its storage complexity is $O(Td)$, which scales linearly with the number of tasks and can become prohibitive for large $T$ or large models $d$.

Pareto Merging (PM) is the most storage-efficient. After its extensive training phase, it produces a single base model augmented with a small set of low-rank parameters that encode the preference-handling mechanism. The storage for these additional parameters is negligible compared to the base model. Thus, its storage complexity is approximately $O(d)$, independent of the number of tasks $T$. While optimal in storage, this efficiency is achieved at a very high computational setup cost, as detailed next.

\subsection{Computational Complexity}
We analyze the computational cost by decoupling it into two phases: a one-time setup phase to prepare the preference-aware framework, and the generation phase where a model is produced for a specific user preference.
\paragraph{One-Time Setup.}  This phase is where the most significant differences between methods emerge. Our approach is analytical and non-iterative. The setup involves two main steps: (1) a single forward pass for each of the $T$ individual models and the base merged model on the $N$ calibration samples to extract representations, with a cost of $O(T N \cdot \text{Cost}(g_\theta))$; and (2) computation of the required matrices, which is dominated by SVD and matrix multiplications, costing $O(T(N D_{\text{rep}}^2 + D_{\text{rep}}^3))$. The entire process is a direct, one-shot computation.
    Conversely, competing methods rely on costly, iterative optimization. MAP's setup involves an expensive data collection stage, where it samples $K$ different merging coefficients, constructs the corresponding models, and evaluates them on all $T$ tasks. This has a complexity of $O(K T \cdot T_{\text{eval}})$, where $T_{\text{eval}}$ is the high cost of a full model evaluation. Subsequently, it fits a quadratic surrogate model, which adds a polynomial cost of $O(K T^5)$. The evaluation stage is the primary bottleneck.
    Pareto Merging's setup consists of a full, end-to-end training procedure using gradient descent. For each of the $T_{\text{steps}}$ training steps, it must sample a preference, compute the loss for all $T$ tasks (requiring $T$ forward passes), and perform a backward pass to update the learnable parameters. This results in a total complexity of $O(T_{\text{steps}} \cdot T \cdot T_{\text{prop}})$, where $T_{\text{prop}}$ is the cost of a forward-backward pass. This process is computationally intensive and can be unstable.
    \paragraph{Inference-Time Preference-Aware Generation.} This phase measures the cost of generating a Pareto-optimal model given a new user preference vector $\mathbf{p}$. Our method is exceptionally fast. It uses the pre-computed matrices to analytically compute the optimal transformation $W_{\mathbf{p}}$ via Eq.~\ref{eq:pareto_cf}. This involves $T$ matrix additions, one matrix inversion, and one matrix multiplication, all on $D_{\text{rep}} \times D_{\text{rep}}$ matrices. The total cost is $O(T D_{\text{rep}}^2 + D_{\text{rep}}^3)$, which is nearly instantaneous in practice.
    Pareto Merging is similarly fast at this stage. Given a preference vector, it computes a low-rank update and adds it to the base model's weights. This involves only small-scale tensor and matrix operations, making the generation extremely efficient.
    MAP is the slowest of the three at this stage. To generate a model for a given preference, it must first find the corresponding merging coefficients $\mathbf{c}^*$ from its pre-computed Pareto front and then construct the final model by performing a weighted sum of all $T$ full-parameter task vectors: $\theta_0 + \sum_t c_t^* \Delta_t$. This operation has a complexity of $O(Td)$, which is substantially more expensive than the other methods, as it involves operations on the entire parameter space.

\section{Additional Experimental Results}
\label{sec:appendix_additional_results}

\subsection{Results for ViT-L/14 Models}
\begin{table*}[t!]
\centering
\tabcolsep=0.30em
\begin{tabular*}{\textwidth}{@{\extracolsep{\fill}}c|l|cccccccc|c}
\toprule
Pref. & Method & SUN397 & Cars & RESISC45 & EuroSAT & SVHN & GTSRB & MNIST & DTD & Average \\
\midrule
\multirow{2}{*}{--} & Individual & 82.3 & 92.4 & 97.4 & 100 & 98.1 & 99.2 & 99.7 & 84.1 & 94.2 \\
    & MTL & 80.8 & 90.6 & 96.3 & 96.3 & 97.6 & 99.1 & 99.6 & 84.4 & 93.5 \\
\midrule
\multirow{5}{*}{--} & TA~\citep{DBLP:conf/iclr/IlharcoRWSHF23} & 73.9 & 82.1 & 86.6 & 94.1 & 87.9 & 86.7 & 98.9 & 65.6 & 84.5 \\
    & TIES~\citep{DBLP:conf/nips/YadavTCRB23} & 75.9 & 85.4 & 89.0 & 95.6 & 89.2 & 87.1 & 99.0 & 68.7 & 86.2 \\
    & AM~\citep{DBLP:conf/iclr/YangW00G0T24} & 79.0 & 90.3 & 90.8 & 96.2 & 93.4 & 98.0 & 99.0 & 79.9 & 90.8 \\
    & AMPP~\citep{DBLP:conf/iclr/YangW00G0T24} & 79.4 & 90.3 & 91.6 & 97.4 & 93.4 & 97.5 & 99.0 & 79.2 & 91.0 \\
\midrule
\multirow{4}{*}{\rotatebox{90}{\textbf{equal}}} & MAP~\citep{liMAPLowcomputeModel2024} & 76.0 & 84.1 & 88.7 & 87.8 & 90.1 & 87.9 & 97.8 & 71.3 & 85.4 \\
   & AM+PM~\citep{chenParetoMerging2025}  & \textbf{79.8} & \textbf{90.7} & 91.4 & 96.5 & 93.5 & 98.3 & 99.1 & 80.6 & 91.2 \\ 
   & \cellcolor{gray!20}AM+Ours& \cellcolor{gray!20}79.6& \cellcolor{gray!20}90.5& \cellcolor{gray!20}\textbf{94.0}& \cellcolor{gray!20}\textbf{97.1}& \cellcolor{gray!20}\textbf{94.3}& \cellcolor{gray!20}\textbf{98.8}& \cellcolor{gray!20}\textbf{99.2}& \cellcolor{gray!20}\textbf{80.9}& \cellcolor{gray!20}\textbf{91.8}\\
   & \cellcolor{gray!20}AM+Ours~($10\%$ unlabeled test data)& \cellcolor{gray!20}79.2& \cellcolor{gray!20}90.3& \cellcolor{gray!20}93.8& \cellcolor{gray!20}96.8& \cellcolor{gray!20}94.1& \cellcolor{gray!20}98.7& \cellcolor{gray!20}99.1& \cellcolor{gray!20}80.3& \cellcolor{gray!20}91.5\\
\midrule
\multirow{3}{*}{\rotatebox{90}{\textbf{priority}}} & AM+PM~\citep{chenParetoMerging2025} & \textbf{80.5} & \textbf{91.7} & 91.9 & \textbf{98.1} & \textbf{96.4} & 98.8 & 99.1 & 80.8 & 92.1 \\
& \cellcolor{gray!20}AM+Ours& \cellcolor{gray!20}80.4& \cellcolor{gray!20}91.1& \cellcolor{gray!20}\textbf{95.3}& \cellcolor{gray!20}97.6& \cellcolor{gray!20}95.2& \cellcolor{gray!20}\textbf{99.1}& \cellcolor{gray!20}\textbf{99.3}& \cellcolor{gray!20}\textbf{83.0}& \cellcolor{gray!20}\textbf{92.6}\\
& \cellcolor{gray!20}AM+Ours~($10\%$ unlabeled test data)& \cellcolor{gray!20}79.8& \cellcolor{gray!20}90.7& \cellcolor{gray!20}95.0& \cellcolor{gray!20}97.3& \cellcolor{gray!20}95.0& \cellcolor{gray!20}98.9& \cellcolor{gray!20}99.3& \cellcolor{gray!20}81.8& \cellcolor{gray!20}92.2\\
\midrule
\multirow{3}{*}{\rotatebox{90}{\textbf{one-hot}}}  
& AM+PM~\citep{chenParetoMerging2025} & - & - & - & - & - & - & - & - & - \\
& \cellcolor{gray!20}AM+Ours& \cellcolor{gray!20}80.7& \cellcolor{gray!20}91.7& \cellcolor{gray!20}96.3& \cellcolor{gray!20}99.1& \cellcolor{gray!20}96.2& \cellcolor{gray!20}99.2& \cellcolor{gray!20}99.6& \cellcolor{gray!20}84.7& \cellcolor{gray!20}93.5\\
& \cellcolor{gray!20}AM+Ours~($10\%$ unlabeled test data)& \cellcolor{gray!20}80.0& \cellcolor{gray!20}91.1& \cellcolor{gray!20}95.4& \cellcolor{gray!20}97.7& \cellcolor{gray!20}95.7& \cellcolor{gray!20}98.8& \cellcolor{gray!20}99.3& \cellcolor{gray!20}80.1& \cellcolor{gray!20}92.3\\
\bottomrule
\end{tabular*}
\caption{Test accuracies (\%) on eight datasets when merging eight ViT-L/14 models. We compare our method against non-controllable baselines and state-of-the-art controllable approaches under three preference scenarios. Our method consistently achieves the highest average accuracy in each group, with the best results highlighted in bold. }
\label{tab:main_comparison_vit-l-14}
\end{table*}

To test our method's scalability and robustness, we extended our experiments to the larger ViT-L/14 model. These results supplement our main findings, showing our method remains effective on different model scales.

Table~\ref{tab:main_comparison_vit-l-14} shows the test accuracies for merging eight tasks on the ViT-L/14 backbone. Consistent with the ViT-B/32 results, our method (AM+Ours) matches or outperforms the state-of-the-art competitor (AM+PM) on average under all preference settings. Our method shows a clear advantage in average accuracy under \textit{equal} and \textit{priority} preferences. It performs especially well in the \textit{one-hot} setting, which requires very precise control. Here, our method achieves 93.5\% average accuracy, proving its ability to meet specific task demands. (We could not reproduce the AM+PM result for the one-hot preference due to limited training resources on training PM modular with 8 ViT-L/14 models.)

Furthermore, our method performs well using only 10\% of the unlabeled data for calibration. In \textit{priority} and \textit{one-hot} settings, it even surpasses the AM+PM baseline which uses the full dataset. These results on ViT-L/14 confirm our method's superiority, data efficiency, and scalability.

\begin{table*}[t!]
\centering
\small
\begin{tabular*}{\textwidth}{@{\extracolsep{\fill}}l|cccccccc|c}
\toprule
\textbf{Method} & \textbf{SUN397} & \textbf{Cars} & \textbf{RESISC45} & \textbf{EuroSAT} & \textbf{SVHN} & \textbf{GTSRB} & \textbf{MNIST} & \textbf{DTD} & \textbf{Average} \\
\midrule
TA~\citep{DBLP:conf/iclr/IlharcoRWSHF23} & 55.2 & 54.9 & 66.7 & 78.9 & 80.2 & 69.7 & 97.3 & 50.4 & 69.1 \\
TA + Ours (equal) & 65.1 & 58.9 & 80.4 & 90.6 & 86.6 & 90.4 & 98.3 & 58.5 & 78.6 \\
TA + Ours (priority) & 67.6 & 64.4 & 87.0 & 95.2 & 89.0 & 95.2 & 98.7 & 69.6 & 83.3 \\
TA + Ours (onehot) & 68.4 & 67.1 & 91.0 & 98.0 & 90.6 & 97.6 & 99.2 & 76.3 & 86.0 \\
\midrule
TIES~\citep{DBLP:conf/nips/YadavTCRB23} & 59.5 & 60.0 & 71.7 & 78.2 & 86.3 & 72.9 & 98.2 & 52.8 & 72.4 \\
TIES + Ours (equal) & 70.4 & 64.5 & 84.1 & 91.3 & 85.6 & 90.7 & 97.8 & 63.1 & 80.9 \\
TIES + Ours (priority) & 72.6 & 68.6 & 89.6 & 95.5 & 87.8 & 95.6 & 98.4 & 72.3 & 85.1 \\
TIES + Ours (onehot) & 73.0 & 71.8 & 92.3 & 98.2 & 89.6 & 97.7 & 99.1 & 78.4 & 87.5 \\
\bottomrule
\end{tabular*}
\caption{Test accuracies (\%) on eight datasets when applying our method to simpler merging backbones. Our post-hoc correction significantly improves the performance of both Task Arithmetic (TA) and TIES-Merging (TIES) under all preference settings. This demonstrates the generality of our framework.}
\label{tab:appendix_other_backbones}
\end{table*}

\begin{figure*}[t!]
    \centering
    \includegraphics[width=0.24\textwidth]{figures/sun397_cars_scatter.pdf}
    \includegraphics[width=0.24\textwidth]{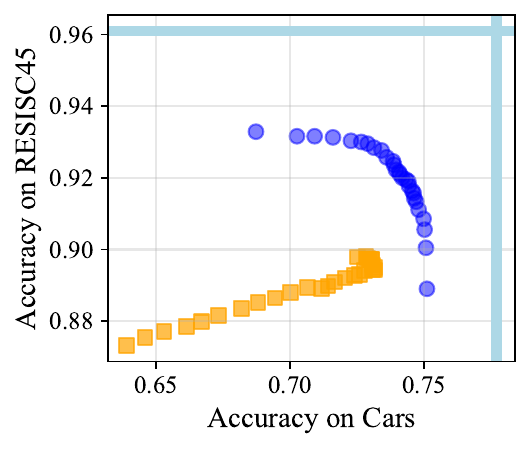}
    \includegraphics[width=0.24\textwidth]{figures/resisc45_eurosat_scatter.pdf}
    \includegraphics[width=0.24\textwidth]{figures/eurosat_svhn_scatter.pdf}
    \includegraphics[width=0.24\textwidth]{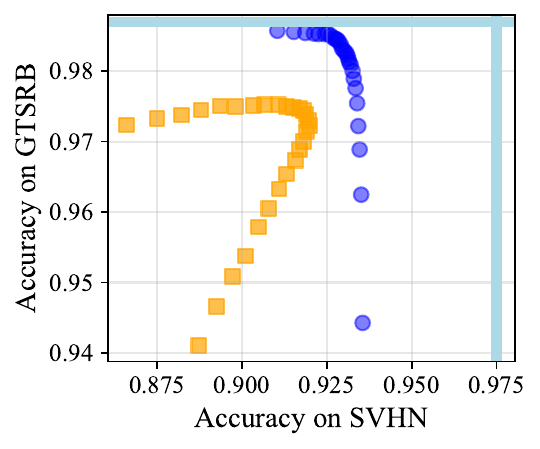}
    \includegraphics[width=0.24\textwidth]{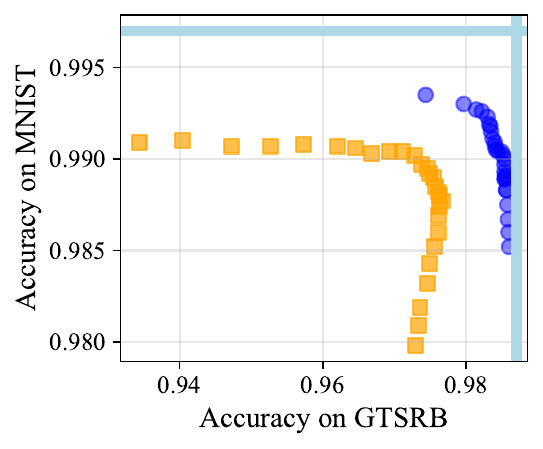}
    \includegraphics[width=0.24\textwidth]{figures/mnist_dtd_scatter.pdf}
    \includegraphics[width=0.24\textwidth]{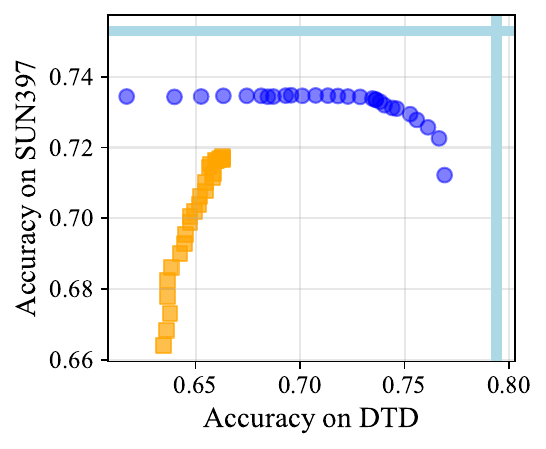}
    \caption{Two task performance under uniformly varying local preferences obtained by AMPP+Ours and AMPP+PM when merging eight ViT-B/32 models. Accuracies on selected task pairs are shown in each subplot.}
    \label{fig:appendix_2task_of_8task}
\end{figure*}

\subsection{Results for Other Pre-merged Models}
To validate the plug-and-play nature of our proposed representation correction framework, we demonstrate its effectiveness as a post-hoc enhancement on various pre-merged models beyond the AdaMerging (AM) and AdaMerging++ (AMPP) backbones used in the main text. Table~\ref{tab:appendix_other_backbones} presents the results of applying our method to simpler merging techniques: Task Arithmetic (TA)~\citep{DBLP:conf/iclr/IlharcoRWSHF23} and TIES-Merging (TIES)~\citep{DBLP:conf/nips/YadavTCRB23}. The results are shown for all preference settings across all eight tasks. Our method consistently provides a significant performance uplift over the base models, showcasing its broad applicability and utility as a general-purpose merging enhancer.

\begin{table*}[t!]
\centering
\begin{tabular*}{\textwidth}{@{\extracolsep{\fill}}l|cccccccc|c}
\toprule
\textbf{Method} & \textbf{SUN397} & \textbf{Cars} & \textbf{RESISC45} & \textbf{EuroSAT} & \textbf{SVHN} & \textbf{GTSRB} & \textbf{MNIST} & \textbf{DTD} & \textbf{Average} \\
\midrule
\multicolumn{10}{l}{\textit{Equal Preference}} \\
\midrule
Naive & 66.1 & \textbf{71.1} & 85.2 & 95.0 & 89.1 & 96.5 & 98.6 & \textbf{68.3} & 83.7 \\
Polar & 67.0 & 70.8 & 85.3 & 91.4 & 88.3 & 95.2 & 98.4 & 68.2 & 83.1 \\
\textbf{Ours} & \textbf{71.0} & 70.1 & \textbf{88.2} & \textbf{95.4} & \textbf{90.9} & \textbf{97.1} & \textbf{98.6} & 68.0 & \textbf{84.9} \\
\midrule
\multicolumn{10}{l}{\textit{Priority Preference (SUN397)}} \\
\midrule
Naive & 71.7 & \textbf{70.2}& 81.7 & 92.5 & 88.6 & 95.5 & 98.4 & \textbf{64.4}& 82.9 \\
Polar & 71.7 & \textbf{70.2}& 81.8 & 89.8 & 87.9 & 94.6 & 98.2 & 64.1 & 82.3 \\
\textbf{Ours} & \textbf{72.5}& 68.6 & \textbf{85.0}& \textbf{93.4}& \textbf{90.2}& \textbf{96.0}& \textbf{98.5}& 63.5 & \textbf{83.5} \\
\midrule
\multicolumn{10}{l}{\textit{Priority Preference (Cars)}} \\
\midrule
Naive & 63.8 & \textbf{75.7}& 83.4 & 92.9 & 87.9 & 94.1 & 98.1 & 62.6 & 82.3 \\
Polar & 65.3 & 74.4 & 83.5 & 91.5 & 87.2 & 93.8 & 98.0 & 62.4 & 82.0 \\
\textbf{Ours} & \textbf{70.4}& 72.8 & \textbf{87.7}& \textbf{95.0}& \textbf{90.6}& \textbf{96.9}& \textbf{98.6}& \textbf{67.0}& \textbf{84.9} \\
\midrule
\multicolumn{10}{l}{\textit{Priority Preference (RESISC45)}} \\
\midrule
Naive & 62.7 & 68.6 & 91.5 & 91.9 & 88.9 & 94.1 & 98.4 & 63.4 & 82.5 \\
Polar & 63.5 & 68.3 & \textbf{92.1}& 88.6 & 87.9 & 91.0 & 98.1 & 63.0 & 81.6 \\
\textbf{Ours} & \textbf{69.5}& \textbf{69.5}& 91.2 & \textbf{93.8}& \textbf{90.6}& \textbf{96.7}& \textbf{98.6}& \textbf{66.3}& \textbf{84.5} \\
\midrule
\multicolumn{10}{l}{\textit{Priority Preference (EuroSAT)}} \\
\midrule
Naive & 58.9 & 66.0 & 73.7 & 98.4 & 89.4 & 94.4 & \textbf{98.6}& 63.9 & 80.4 \\
Polar & 54.4 & 57.4 & 68.7 & \textbf{98.5} & 88.4 & 85.3 & 98.1 & 56.6 & 75.9 \\
\textbf{Ours} & \textbf{70.6}& \textbf{69.7}& \textbf{87.0}& 97.1 & \textbf{90.8}& \textbf{96.9}& \textbf{98.6}& \textbf{66.8}& \textbf{84.7} \\
\midrule
\multicolumn{10}{l}{\textit{Priority Preference (SVHN)}} \\
\midrule
Naive & 56.4 & 56.6 & 80.8 & 91.7 & 92.1 & 93.7 & 98.4 & 62.3 & 79.0 \\
Polar & 53.2 & 52.7 & 76.2 & 86.1 & \textbf{92.8}& 86.2 & 98.2 & 57.7 & 75.4 \\
\textbf{Ours} & \textbf{69.7}& \textbf{68.5}& \textbf{87.5}& \textbf{94.9}& 92.3 & \textbf{95.7}& \textbf{98.5}& \textbf{65.9}& \textbf{84.1} \\
\midrule
\multicolumn{10}{l}{\textit{Priority Preference (GTSRB)}} \\
\midrule
Naive & 56.4 & 60.1 & 80.2 & 91.3 & 86.7 & \textbf{98.4}& 97.8 & 59.6 & 78.8 \\
Polar & 54.7 & 57.4 & 78.8 & 86.0 & 84.2 & \textbf{98.4}& 96.2 & 57.3 & 76.6 \\
\textbf{Ours} & \textbf{69.8}& \textbf{69.0}& \textbf{87.7}& \textbf{94.7}& \textbf{90.4}& 98.1 & \textbf{98.5}& \textbf{66.4}& \textbf{84.3} \\
\midrule
\multicolumn{10}{l}{\textit{Priority Preference (MNIST)}} \\
\midrule
Naive & 57.4 & 60.7 & 82.6 & 92.1 & 88.1 & 93.1 & 99.1 & 61.8 & 79.3 \\
Polar & 49.8 & 55.4 & 78.1 & 87.9 & 87.7 & 87.7 & \textbf{99.2}& 55.1 & 75.1 \\
\textbf{Ours} & \textbf{70.6}& \textbf{69.5}& \textbf{88.0}& \textbf{95.4}& \textbf{90.5}& \textbf{96.7}& 98.8 & \textbf{67.1}& \textbf{84.6} \\
\midrule
\multicolumn{10}{l}{\textit{Priority Preference (DTD)}} \\
\midrule
Naive & 62.3 & 66.2 & 83.5 & 91.8 & 88.4 & 94.0 & 98.2 & 76.8 & 82.6 \\
Polar & 62.3 & 65.8 & 83.4 & 85.8 & 87.5 & 92.9 & 97.7 & \textbf{77.2}& 81.6 \\
\textbf{Ours} & \textbf{70.5}& \textbf{69.8}& \textbf{87.9}& \textbf{95.0}& \textbf{90.7}& \textbf{96.9}& \textbf{98.6}& 73.7 & \textbf{85.4} \\
\bottomrule
\end{tabular*}
\caption{Full test accuracies (\%) for different aggregation strategies on the AM backbone. We compare our proposed method against the Naive and Polar baselines under equal and priority preference settings. Our method consistently achieves a better balance, leading to a higher average performance.}
\label{tab:appendix_aggregation_full}
\end{table*}

\subsection{Additional Performance on Selected Task Pairs}
Figure~\ref{fig:appendix_2task_of_8task} presents two task performance under uniformly varying local preferences obtained by AMPP+Ours and AMPP+PM when merging eight ViT-B/32 models. The results demonstrate that our method achieves superior Pareto frontiers compared to PM on the shown task pairs. Notably, all response scatters exhibit smooth transitions corresponding to uniformly varying preferences between the two target tasks, while the other six tasks are disabled (preference weight = 0), whereas PM fails to produce controllable responses for critical task pairs (SUN397-Cars, Cars-RESISC45, MNIST-DTD, and DTD-SUN397), highlighting our method's enhanced controllability.

\begin{figure*}[t!]
    \centering
\includegraphics[width=0.35\textwidth]{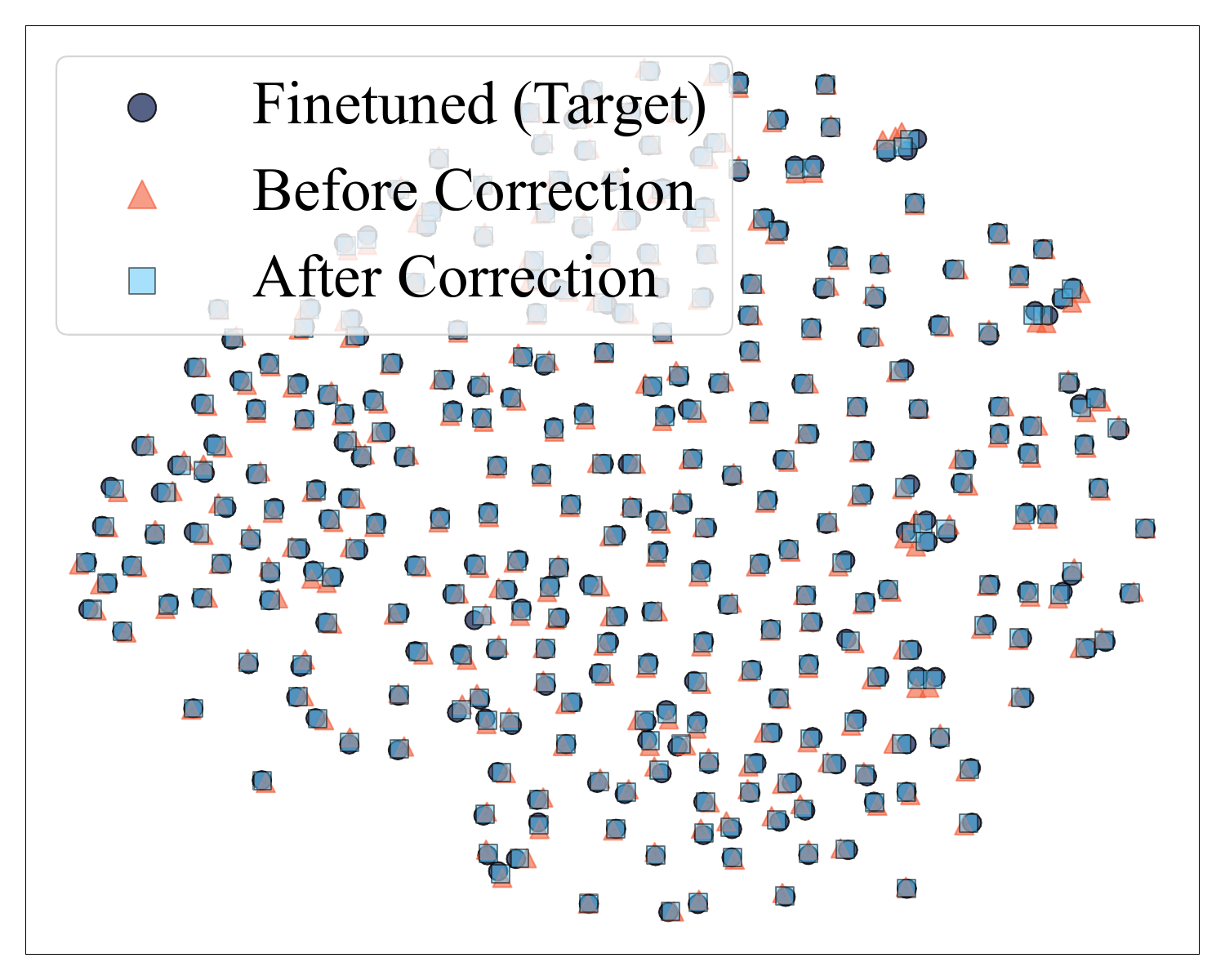}
    \includegraphics[width=0.35\textwidth]{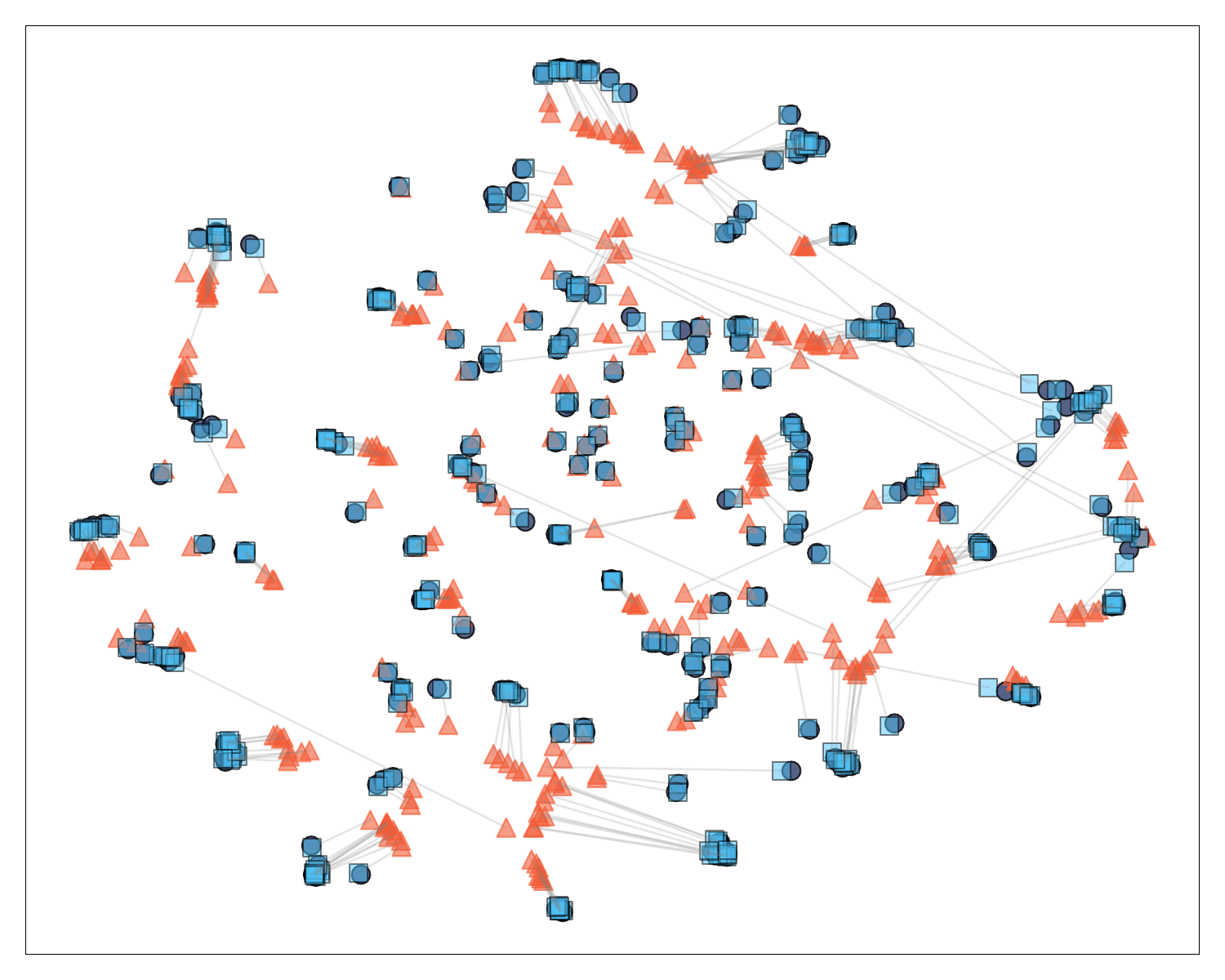}
    \includegraphics[width=0.35\textwidth]{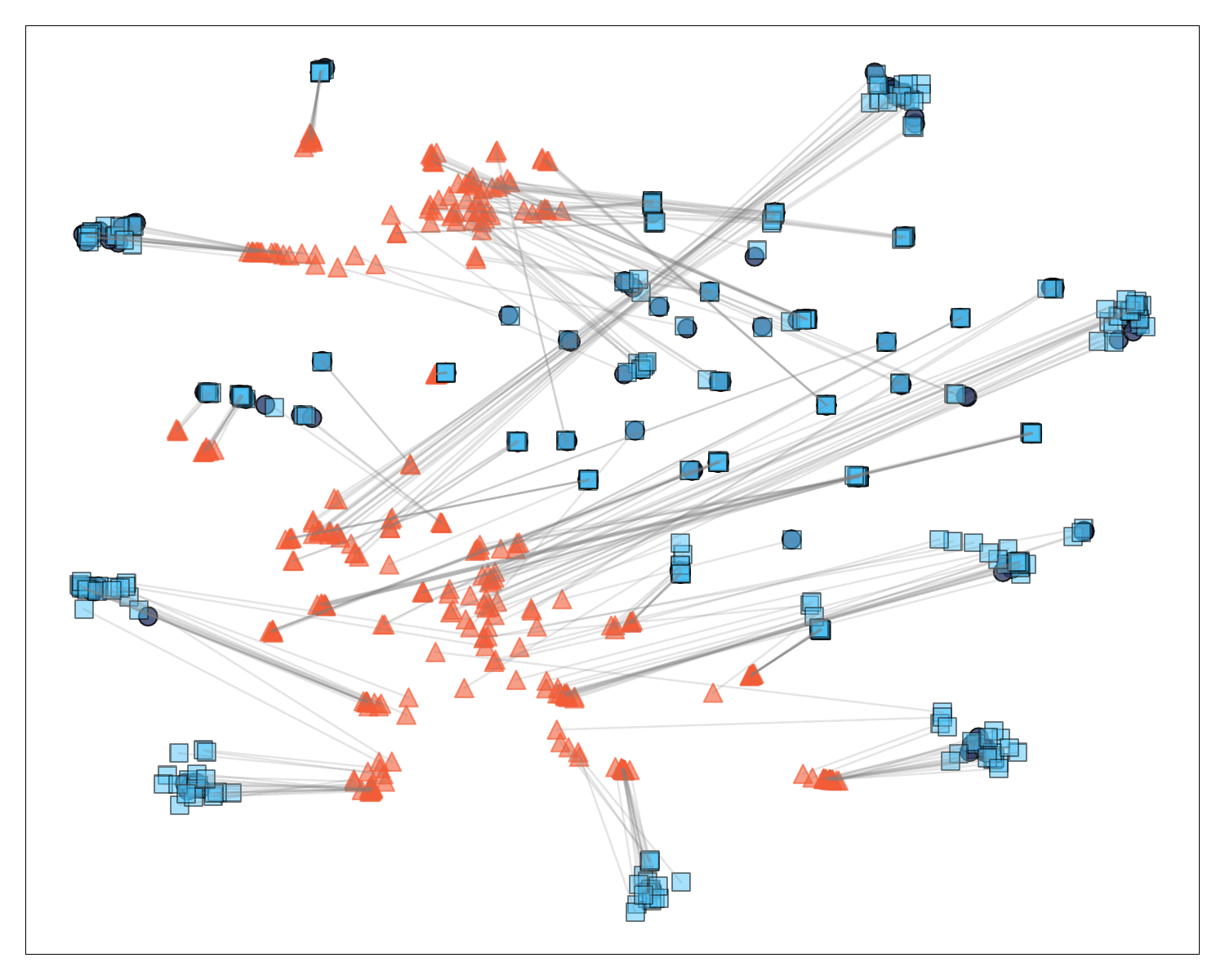}
    \includegraphics[width=0.35\textwidth]{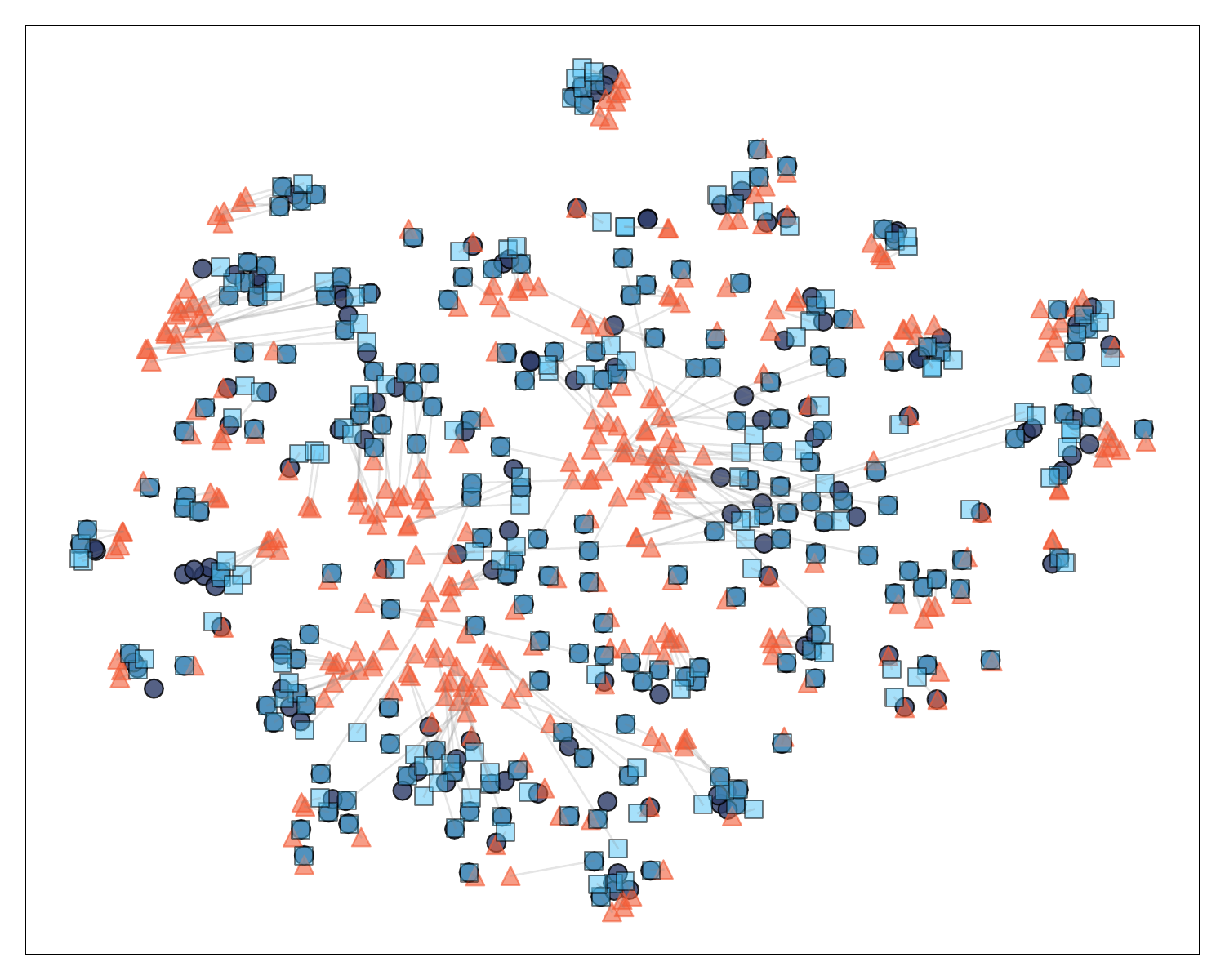}
    \caption{We visualize the 2D t-SNE embeddings of representations for four additional tasks: SUN397, RESISC45, GTSRB, and DTD (from top left to bottom right). For each task, we show the representations from the pre-merged model ($\mathcal{Z}^\text{mtl}$, blue), the ideal single-task model ($\mathcal{Z}^\text{ind}$, orange), and our corrected model ($\hat{\mathcal{Z}}^\text{mtl}$, green). Across all tasks, our method successfully aligns the distorted merged representations with the ideal ones via a single linear transformation, corroborating our core hypothesis.}
    \label{fig:appendix_tsne}
\end{figure*}

\subsection{Full Results for Aggregation Baselines}
To supplement the analysis of aggregation strategies, we provide the full, absolute accuracy results for the aggregation baselines discussed. In the main paper (Table~\ref{tab:weighting_comparison}), we used normalized metrics to compare our principled Pareto-optimal aggregation against the \textit{Naive} and \textit{Polar} heuristics. Table~\ref{tab:appendix_aggregation_full} below presents the raw test accuracies on all eight tasks for these methods under both equal and priority preference settings. These detailed results allow for a granular inspection of how each strategy trades off performance across tasks and substantiates our main claim that our proposed aggregation yields a more globally robust solution.

\subsection{Additional Visualizations}
Figure~\ref{fig:appendix_tsne} extends the t-SNE visualization from the main paper to four additional datasets: SUN397, RESISC45, GTSRB, and DTD. In each case, the representation of the pre-merged model ($\mathcal{Z}^\text{mtl}$) exhibits a distributional misalignment compared to the ideal single-task representation ($\mathcal{Z}^\text{ind}$). Our simple linear correction effectively aligns the merged representation with the ideal one, visually confirming that a linear map is sufficient to resolve the primary discrepancy across a diverse set of tasks.

\section{Limitations and Future Work} 
\label{sec:appendix_limitations}

Our method assumes that the changes in model representations caused by merging can be corrected with a linear transformation. This works well in our experiments, but may not hold when the changes are highly non-linear. Additionally, our approach leverages linear scalarization to derive its closed-form solution. While highly efficient, this inherently means that if the true Pareto front of final task performance exhibits concave regions, our method may not be able to identify solutions within those regions. This is a known theoretical limitation of linear scalarization in multi-objective optimization, representing a trade-off for our method's exceptional speed and analytical nature. Our method also needs an unlabeled dataset from the target task to learn the correction. In the future, it would be useful to build correction methods that do not require any target data, by only analyzing the model weights.

So far, we have tested our method on vision models like ViT. We did not test it on larger language models or models that handle both images and text. These models may have more complex representation structures, and merging them could create new challenges. Extending our approach to such models is an important direction for future work.
 
\end{document}